\documentclass[10pt,twocolumn,letterpaper]{article}

\usepackage{titling}
\usepackage{iccv}
\usepackage{times}
\usepackage{epsfig}
\usepackage{graphicx}
\usepackage{amsmath}
\usepackage{amssymb}
\usepackage{amsthm}
\usepackage{multirow}
\usepackage{svg}
\usepackage{contour}
\usepackage{color}
\usepackage{fancyhdr}
\usepackage{setspace}

\graphicspath{{./media/}}

\usepackage{framed}
\newenvironment{myequ}[1]
   {\begin{framed}{\noindent\bfseries#1}%
     \addtolength\textwidth{20pt}\par\noindent\hspace*{-10pt}\rule{\textwidth}{.4pt}}
   {\end{framed}}

\usepackage{siunitx}
\newtheorem{theorem}{Theorem}[section]
\newtheorem{proposition}[theorem]{Proposition}
\newtheorem{remark}[theorem]{Remark}
\newcommand{\PAR}[1]{\vskip4pt \noindent{\bf #1~}}

\usepackage[pagebackref=true,breaklinks=true,letterpaper=true,colorlinks,bookmarks=false]{hyperref}

\newif\ifarxiv
\arxivtrue

\ifarxiv
    \iccvfinalcopy
\fi

\def\eg{\emph{e.g}\onedot} 
\def\ie{\emph{i.e}\onedot} 
\newcommand{\TODO}[1]{\textcolor{red}{#1}}
\newcommand{\refFig}[1]{Figure~\ref{#1}}

\newcommand{\refEq}[1]{Equation~(\ref{#1})}
\newcommand{\refTab}[1]{Table~\ref{#1}}

\def\iccvPaperID{919} 

\ificcvfinal\fi
\begin{document}

\title{Learning Proximal Operators: \\ Using Denoising Networks for Regularizing Inverse Imaging Problems}

\author{Tim Meinhardt$^1$\\
{\tt\scriptsize tim.meinhardt@tum.de}
\and Michael Moeller$^2$\\
{\tt\scriptsize michael.moeller@uni-siegen.de}
\and Caner Hazirbas$^1$\\
{\tt\scriptsize hazirbas@cs.tum.edu}
\and Daniel Cremers$^1$\\
{\tt\scriptsize cremers@tum.de} \\ \and \vspace{0.2cm}
\normalfont{Technical University of Munich$^1$ \quad University of Siegen$^2$} 
}

\ifarxiv
    \date{}
\fi

\maketitle
\thispagestyle{fancy}

\begin{abstract}
While variational methods have been among the most powerful tools for solving linear inverse problems in imaging, deep (convolutional) neural networks have recently taken the lead in many challenging benchmarks. A remaining drawback of deep learning approaches is their requirement for an expensive retraining whenever the specific problem, the noise level, noise type, or desired measure of fidelity changes. On the contrary, variational methods have a plug-and-play nature as they usually consist of separate data fidelity and regularization terms. 

In this paper we study the possibility of replacing the proximal operator of the regularization used in many convex energy minimization algorithms by a denoising neural network. The latter therefore serves as an implicit natural image prior, while the data term can still be chosen independently. Using a fixed denoising neural network in exemplary problems of image deconvolution with different blur kernels and image demosaicking, we obtain state-of-the-art reconstruction results. These indicate the high generalizability of our approach and a reduction of the need for problem-specific training. Additionally, we discuss novel results on the analysis of possible optimization algorithms to incorporate the network into, as well as the choices of algorithm parameters and their relation to the noise level the neural network is trained on.

\end{abstract}

\begin{figure}
\includegraphics[width=\linewidth]{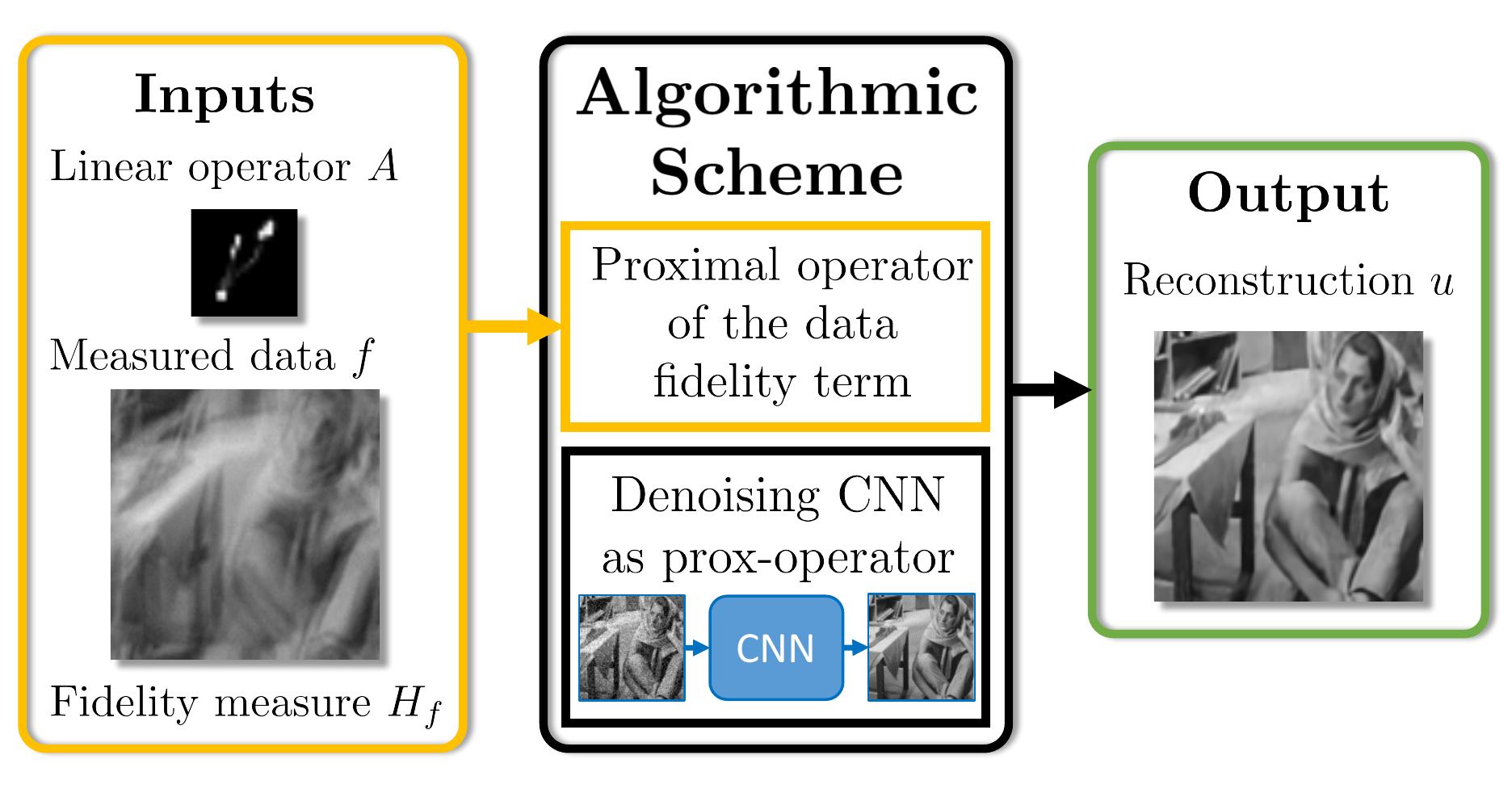}
\caption{We propose to exploit the recent advances in convolutional neural networks for image denoising for general inverse imaging problems by replacing the proximal operator in optimization algorithms with such a network. Changing the image reconstruction task, e.g. from deblurring to demosaicking, merely changes the data fidelity term such that the same network can be used over a wide range of applications without requiring any retraining.}
\label{fig:teaser}
\vspace{-0.27cm}
\end{figure}

\section{Introduction}
Many important problems in image processing and computer vision can be phrased as linear inverse problems where the desired quantity $u$ cannot be observed directly but needs to be determined from measurements $f$ that relate to $u$ via a linear operator $A$,~\ie $f = Au + n$ for some noise $n$. In almost all practically relevant applications the solution is very sensitive to the input data, and the underlying continuous problem is \text{ill-posed}. A classical but powerful general approach to obtain stable and faithful reconstructions is to use a \textit{regularization} and determine the estimated solution $\hat{u}$ via an energy minimization problem of the form
\begin{align}
    \hat{u} = \text{argmin}_u H_f(Au) + R(u).
    \label{eq:variational_method}
\end{align}
In the above, $H_f$ is a fidelity measure that relates the data $f$ to the estimated true solution $u$,~\eg $H_f(Au)=\|Au-f\|^2$ and $R$ is a regularization function that introduces a-priori information on the expected solution.

Recently, the computer vision research community has had great success in replacing the explicit modeling of energy functions in \refEq{eq:variational_method} by parameterized functions $\mathcal{G}$ that directly map the input data $f$ to a solution $\hat{u} = \mathcal{G}(f)$. Powerful architectures are so-called \textit{deep networks} that parameterize $\mathcal{G}$ by several layers of linear operations followed by certain nonlinearities,~\eg rectified linear units. The free parameters of $\mathcal{G}$ are \textit{learned} by using large amounts of training data and fitting the parameters to the ground truth data via a large-scale optimization problem.

Deep networks have had a big impact in many fields of computer vision. Starting from the first large-scale applications of convolutional neural networks (CNNs),~\eg ImageNet classification~\cite{krizhevsky12_alexnet,simonyan15_vgg,he16_resnet}, deep networks have recently been extended to high dimensional inverse problems such as image denoising~\cite{xie12_denoising,zhang16_denoising}, deblurring~\cite{xu14_deconvolution}, super-resolution~\cite{dong14_srcnn, dong16_srcnn}, optical flow estimation~\cite{dosovitsky15_flownet, mayer16_dataset}, image demosaicking~\cite{wang14_demosaicking,gharbi16_demosaicking,klatzer16_joint_demo_denoi}, or inpainting~\cite{kohler2014_mask,yeh16_inpainting}. In many cases, the performance of deep networks can be further improved when the prediction of the network is postprocessed with an energy minimization method,~\eg optical flow~\cite{guney16_deep} and stereo matching (disparity estimation)~\cite{zbontar16_stereo, chen15_stereo, luo15_stereo}.

While learning based methods yield powerful representations and are efficient in the evaluation of the network for given input data $f$, their training is often difficult. A sufficient amount of training data needs to be acquired in such a way that it generalizes well enough to the test data the network is finally used for. Furthermore, the final performance often depends on a required training and network architecture expertise which includes weight regularization~\cite{krogh92_wdecay}, dropout~\cite{srivastava14_dropout}, batch normalization~\cite{ioffe15_batchnorm}, or the introduction of ``shortcuts''~\cite{he16_resnet}. Finally, while it is very quick and easy to change the linear operator $A$ in variational methods like~\refEq{eq:variational_method}, learning based methods require a costly training as soon as the operator $A$ changes. The latter motivates the idea to combine the advantages of energy minimization methods that are flexible to changes of the data term with the powerful representation of natural images that can be obtained via deep learning.

It was observed in~\cite{venkatakrishnan_13_ppp, heide14_flexisp} that modern convex optimization algorithms for solving \refEq{eq:variational_method} merely depend on the proximal operator of the regularization $R$, which motivated the authors to replace this step by general designed denoising algorithms such as the non-local means (NLM)~\cite{buades11_nlm} or BM3D~\cite{dabov07_bm3d} algorithms. Upon preparation of this manuscript we additionally found the ArXiv report~\cite{romano16_red} which extends the ideas of~\cite{venkatakrishnan_13_ppp} and offers a detailed theoretical analysis on solving linear inverse problems by turning them into a chain of denoising steps. For the sake of completeness, we have to mention methods such as~\cite{wang16_deep_structured_models} who apply the contrary approach and use variational methods as boilerplate models to design their network architecture.

In this paper we exploit the power of learned image denoising networks by using them to replace the proximal operators in convex optimization algorithms as illustrated in \refFig{fig:teaser}. Our contributions are:
\begin{itemize}
	\item We demonstrate that using a fixed denoising network as a proximal operator in the primal-dual hybrid gradient (PDHG) method yields state-of-the-art results close to the performance of methods that trained a problem-specific network. 
    \item We analyze the possibility to use different optimization algorithms for incorporating neural networks and show that the fixed points of the resulting algorithmic schemes coincide.
    \item We provide new insights about how the final result is influenced by the algorithm's step size parameter and the denoising strength of the neural network.
\end{itemize}
%
%
\section{Related work}
Classical variational methods exploiting~\refEq{eq:variational_method}, use regularization functions that are designed to suppress noise while preserving important image features. One of the most famous examples is the total variation (TV)~\cite{rof92_tv} which penalizes the norm of the gradient of an image and has been shown to preserve image discontinuities.

An interesting observation is that typical convex optimization methods for~\refEq{eq:variational_method} merely require the evaluation of the \textit{proximal operator} of the regularization functional $R$,
\begin{align}
    \text{prox}_{R}(b) = \text{argmin}_u \frac{1}{2}\|u-b\|_2^2 + R(u).
    \label{eq:proximal_operator}
\end{align}
The interpretation of the proximal operator as a denoising of $b$ motivated the authors of~\cite{venkatakrishnan_13_ppp, heide14_flexisp} to replace the proximal operator of $R$ by a powerful denoising method such as NLM or BM3D. 
Theoretical results including conditions under which the alternating directions method of multipliers (ADMM) with a custom proximal operator converges were presented in \cite{Rond16,Chan17}. 

Techniques using customized proximal operators have recently been explored in several applications, e.g. Poisson denoising \cite{Rond16}, bright field electron tomography \cite{Sreehari16}, super-resolution \cite{Brifman16}, or hyperspectral image sharpening \cite{Teodoro2017}. Interestingly, the aforementioned works all focused on patch-based denoising methods as proximial operators. While \cite{Teodoro16} included a learning of a Gaussian mixture model of patches, we propose to use deep convolutional denoising networks as proximal operators, and analyze their behavior numerically as well as theoretically. 





\section{Learned proximal operators}
    \subsection{Motivation via MAP estimates}
    A common strategy to motivate variational methods like~\refEq{eq:variational_method} are maximum a-posteriori probability (MAP) estimates. One desires to maximize the conditional probability $p(u|f)$ that $u$ is the true solution given that $f$ is the observed data. One applies Bayes rule, and minimizes the negative logarithm of the resulting expression to find
    \begin{align}
    \arg\max_u ~& p(u|f) = \arg\min_u -\log\left(\frac{p(f|u) p(u)}{p(f)} \right) \\
    =& \arg\min_u\left( -\log(p(f|u)) - \log(p(u))\right).
    \end{align}
In the light of MAP estimates, the data term is well described by the forward operator $A$ and the assumed noise model. For example, if the observed data $f$ differs from the true data $Au$ by Gaussian noise of variance $\sigma^2$, it holds that $p(f|u) = \exp(-\frac{\|Au-f\|_2^2}{2\sigma^2})$, which naturally yields a squared $\ell^2$ norm as a data fidelity term.  Therefore, having a good estimate on the forward operator $A$ and the underlying noise model seems to make ``learning the data term'' obsolete.

    A much more delicate term is the regularization, which -- in the framework of MAP estimates -- corresponds to the negative logarithm of the probability of observing $u$ as an image. Assigning a probability to any possible $\mathbb{R}^{n \times m}$ matrix that could represent an image, seems extremely difficult by simple, hand-crafted measures. Although penalties like the TV are well-motivated in a continuous setting, the norm of the gradient cannot fully capture the likelihood of complex natural images. Hence, the regularization is the perfect candidate to be replaced by learning-based techniques. 

    \subsection{Algorithms for learned proximal operators}
    Motivated by MAP estimates ``learning the probability $p(u)$ of natural images'', seems to be a very attractive strategy. As learning $p(u)$ directly appears to be difficult from a practical point of view, we instead
exploit the observation of \cite{venkatakrishnan_13_ppp,heide14_flexisp} that many convex optimization algorithms for~\refEq{eq:variational_method} only require the proximal operator of the regularization.

For instance, applying a proximal gradient (PG) method to the minimization problem in~\refEq{eq:variational_method} yields the update equation
    \begin{align}
    \label{eq:proxGrad}
    u^{k+1} = \text{prox}_{\tau R}\left(u^k - \tau A^* \nabla H_f(Au^k) \right).
    \end{align}
    Since a proximal operator can be interpreted as a Gaussian denoiser in a MAP sense, an interesting idea is to replace the above proximal operator of the regularizer by a neural network $\mathcal{G}$,~\ie
    \begin{align}
    \label{eq:proxGradNeuralNet}
    u^{k+1} = \mathcal{G}\left(u^k - \tau A^* \nabla H_f(Au^k) \right).
    \end{align}
    Instead of the proximal gradient method in~\refEq{eq:proxGrad}, the plug-and-play priors considered in~\cite{venkatakrishnan_13_ppp} utilize the ADMM algorithm leading to update equations of the form
    
    \begin{align}
    \label{eq:admm1}
    u^{k+1} =& \text{prox}_{\frac{1 }{\gamma} (H_f \circ A)}\left(v^{k+1}-\frac{1 }{\gamma}y^{k}\right), \\
            \label{eq:admm2}
    v^{k+1} =& \text{prox}_{\frac{1 }{\gamma} R}\left(u^{k}+\frac{1 }{\gamma}y^k\right), \\
    \label{eq:admm3}
    y^{k+1} =& y^k + \gamma(u^{k+1} - v^{k+1}),
    \end{align}

    and consider replacing the proximal operator in~\refEq{eq:admm2} by a general denoising method such as NLM or BM3D. Replacing~\refEq{eq:admm2} by a neural network can be motivated equally.

    Finally, the authors of~\cite{heide14_flexisp} additionally consider a purely primal formulation of the primal-dual hybrid gradient method (PDHG)~\cite{PCBC-ICCV09,Esser-Zhang-Chan-10,chambollePock11}. For~\refEq{eq:variational_method} such a method amounts to update equations of the form
    \begin{align}
    \label{eq:pdhg1}
    z^{k+1} =& z^{k} + \gamma A\bar{u}^k - \gamma \text{prox}_{\frac{1}{\gamma}H_f}\left(\frac{1}{\gamma} z^{k} + A\bar{u}^k\right), \\
    \label{eq:pdhg2}
    y^{k+1} =& y^{k} + \gamma \bar{u}^k - \gamma \text{prox}_{\frac{1 }{\gamma}R}\left(\frac{1}{\gamma} y^{k} + \bar{u}^k\right), \\
    \label{eq:pdhg3}
    u^{k+1} =& u^k - \tau A^T z^{k+1} - \tau y^{k+1} ,\\
        \label{eq:pdhg4}
    \bar{u}^{k+1}=& u^{k+1} + \theta(u^{k+1} -u^{k}),
    \end{align}
if $\text{prox}_{H_f \circ A}$ is difficult to compute, or otherwise 
  \begin{align}
    \label{eq:pdhg2b}
    y^{k+1} =& y^{k} + \gamma \bar{u}^k - \gamma \text{prox}_{\frac{1 }{\gamma}R}\left(\frac{1}{\gamma} y^{k} + \bar{u}^k\right), \\
    \label{eq:pdhg3b}
    u^{k+1} =& \text{prox}_{\tau (H_f \circ A)}(u^k - \tau y^{k+1}) ,\\
    \label{eq:pdhg4b}
    \bar{u}^{k+1}=& u^{k+1} + \theta(u^{k+1} -u^{k}).
    \end{align}
In both variants of the PDHG method shown above, linear operators in the regularization (such as the gradient in case of TV regularization) can further be decoupled from the computation of the remaining proximity operator. From now on we will refer to \eqref{eq:pdhg1}--\eqref{eq:pdhg4} as PDHG1 and to \eqref{eq:pdhg2b}--\eqref{eq:pdhg4b} as PDHG2. 

 Again, the authors of~\cite{heide14_flexisp} considered replacing the proximal operator in update~\refEq{eq:pdhg2} or~\refEq{eq:pdhg2b} by a BM3D or NLM denoiser, which -- again -- motivates replacing such a designed algorithm by a learned network $\mathcal{G}$,~\ie
    \begin{align}
    \label{eq:pdhg2NN}
    y^{k+1}= y^{k} + \gamma \bar{u}^k - \gamma~ \mathcal{G}\left(\frac{1}{\gamma} y^{k} + \bar{u}^k\right).
    \end{align}
    A natural question is which of the algorithms PG, ADMM, PDHG1, or PDHG2 should be used together with a denoising neural network? The convergence of any of the four algorithms can only be guaranteed for sufficiently friendly convex functions, or in some nonconvex settings under specific additional assumptions. The latter is an active field of research such that analyzing the convergence even beyond nonconvex functions goes beyond the scope of this paper. We refer the reader to \cite{Rond16,Chan17} for some results on the convergence of ADMM with customized proximal operators. 
    
    We will refer to the proposed method as an \textit{algorithmic scheme} in order to indicate that a proximal operator has been replaced by a denoising network. Despite this heuristics, our numerical experiments as well as previous publications indicate that the modified iterations remain stable and converge in a wide variety of cases. Therefore, we investigate the fixed-points of the considered schemes. Interestingly, the following remark shows that the set of fixed-points does not differ for different algorithms.
    \begin{remark}
    \label{prop:stationaryPoints}
    Consider replacing the proximal operator of $R$ in the PG, ADMM, PDHG1, and PDHG2 methods by an arbitrary continuous function $\mathcal{G}$. Then the fixed-point equations of all four resulting algorithmic schemes are equivalent, and yield  
%
    \begin{align}
    \label{eq:fixedPoint}
     u_* = \mathcal{G}\left(u_* - t A^T\nabla H_f (Au_*)\right)
    \end{align}
    with $* \in \{ \text{PG}, \text{ADMM}, \text{PDHG1}, \text{PDHG2}\}$ and $t=\tau$ for PG and PDHG2, and $t=\frac{1}{\gamma}$ for ADMM and PDHG1. 
\end{remark}
    \begin{proof}
    See supplementary material. 
    \end{proof}

\subsection{Parameters for learned proximal operators}    
    One key question when replacing a proximity operator of the form $\text{prox}_{\frac{1}{\gamma} R}$  by a Gaussian denoising operator, is the relation between the step size $\gamma$ and the noise standard deviation $\sigma$ used for the denoiser. Note that $\text{prox}_{\frac{1}{\gamma} R}$ can be interpreted as a MAP estimate for removing zero-mean Gaussian noise with standard-deviation $\sigma = \sqrt{\gamma}$ (as also shown in~\cite{venkatakrishnan_13_ppp}). Therefore, the authors of~\cite{heide14_flexisp} used the PDHG algorithm with a BM3D method as a proximal operator in~\refEq{eq:pdhg2b} and adopted the BM3D denoising strength according to the relation $\sigma = \sqrt{\gamma}$. While algorithms like BM3D allow to easily choose the denoising strength, a neural network is less flexible as an expensive training is required for each choice of denoising strength. 

An interesting insight can be gained by using the algorithmic scheme arising from the PDHG2 algorithm with stepsize $\tau = \frac{c}{\gamma}$ for some constant $c$, and the proximity operator of the regularization being replaced by an arbitrary function $\mathcal{G}$. In the case of convex optimization, i.e. the original PDHG2 algorithm, the constant $c$ resembles the stability condition that $\tau \gamma$ has to be smaller than the squared norm of the involved linear operator. After using $\mathcal{G}$ instead of the proximal mapping, the resulting algorithmic scheme becomes
 \begin{align}
 \label{eq:pdhg2c}
     y^{k+1} =& y^{k} + \gamma \bar{u}^k - \gamma~ \mathcal{G}\left(\frac{1}{\gamma} y^{k} + \bar{u}^k\right),\\
      \label{eq:pdhg3c}
    u^{k+1} =& \text{prox}_{\frac{c}{\gamma} (H_f \circ A)}(u^k - \frac{c}{\gamma} y^{k+1}) ,\\
      \label{eq:pdhg4c}
    \bar{u}^{k+1}=& u^{k+1} + \theta(u^{k+1} -u^{k}).
    \end{align}
We can draw the following simple conclusion:
\begin{proposition}
\label{prop:stepsizeInvariance}
Consider the algorithmic scheme given by Equations \eqref{eq:pdhg2c}--\eqref{eq:pdhg4c}. Then any choice of $\gamma>0$ is equivalent to $\gamma = 1$ with a newly weighted data fidelity term $\tilde{H}_f = \frac{1}{\gamma} H_f$. In other words, changing the step size $\gamma$ merely changes the data fidelity parameter.
\end{proposition}
\begin{proof}
We divide~\refEq{eq:pdhg2c} by $\gamma$ and define $\tilde{y}^k = \frac{1}{\gamma}y^k$. The resulting algorithm becomes
\begin{align}
     \tilde{y}^{k+1} =& \tilde{y}^{k} + \bar{u}^k - \mathcal{G}\left(\tilde{y}^k + \bar{u}^k\right),\\
    u^{k+1} =& \text{prox}_{c (\tilde{H}_f \circ A)}(u^k - c ~\tilde{y}^{k+1}) ,\\
    \bar{u}^{k+1}=& u^{k+1} + \theta(u^{k+1} -u^{k}),
    \end{align}
which yields the assertion. 
\end{proof}
We'd like to point out that Proposition~\ref{prop:stepsizeInvariance} states the equivalence of the update equations. For the iterates to coincide one additionally needs the initialization $y^0=0$. 

Interestingly, similar results can be obtained for any of the four schemes discussed above. As a conclusion, the specific choice of the step sizes $\tau$ and $\sigma$ does not matter, as they simply rescale the data fidelity term, which should have a free tuning parameter anyway. 


Besides the step sizes $\tau$ and $\sigma$, an interesting question is how the denoising strength of a neural network $\mathcal{G}$ relates to the data fidelity parameter. In analogy to MAP estimates above, one could expect that increasing the standard deviation $\sigma$ of the noise the network is trained on by a factor of $a$, requires the increase of the data fidelity parameter by a factor of $a^2$ in order to obtain equally optimal results. 

\begin{figure}[ht!]
\centering
\includegraphics[width=0.95\linewidth]{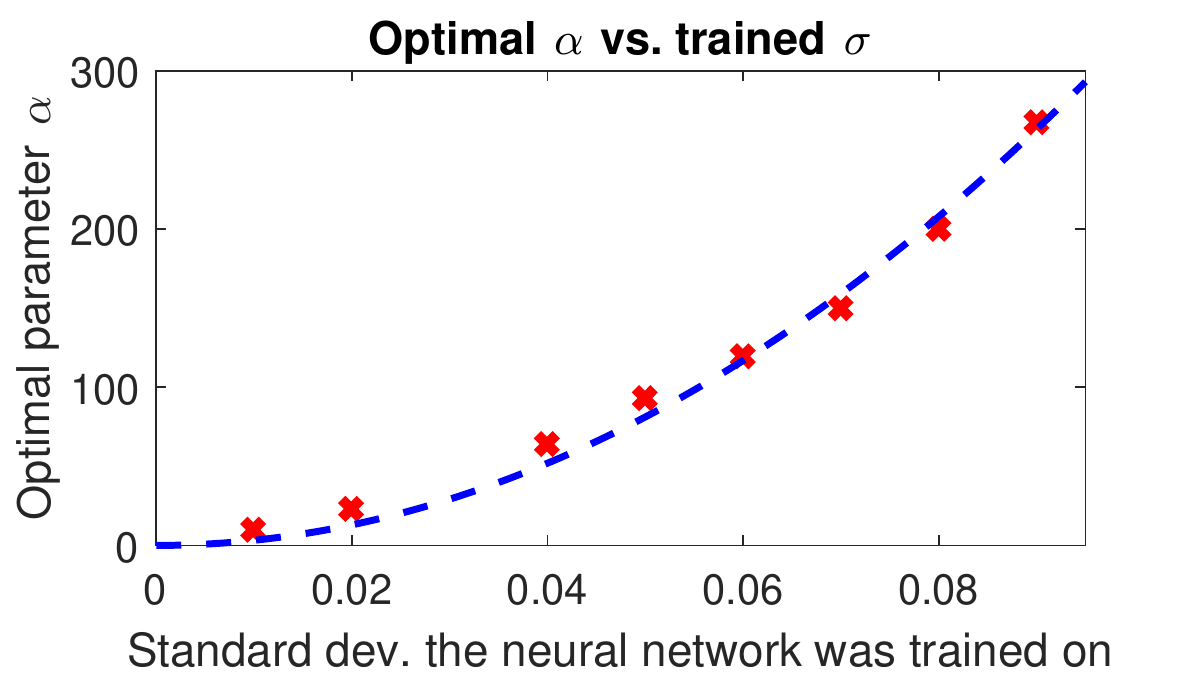}\\\vspace{-0.75cm}

\includegraphics[width=0.95\linewidth]{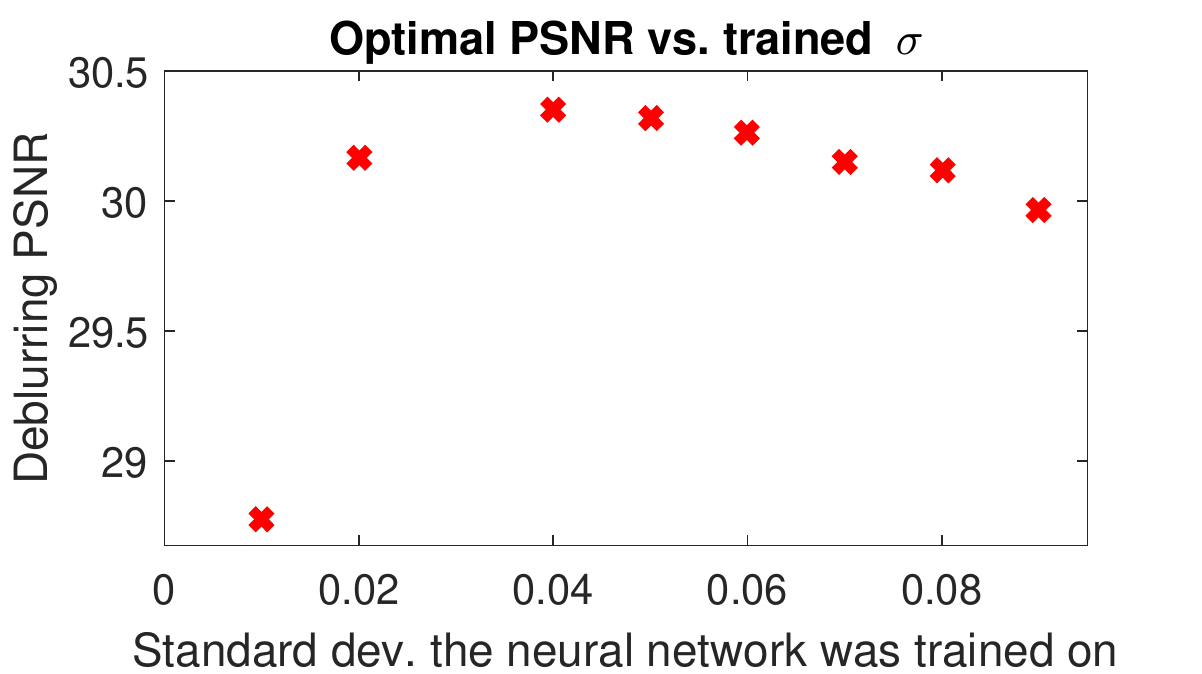}
\caption{The same deconvolution experiment was run with denoising networks trained on noise with different standard deviations $\sigma$ as proximal operators. The first plot shows the optimal data fidelity parameter $\alpha$ as a function of $\sigma$ and the dashed blue curve is the best quadratic fit. It verifies the expected theoretical quadratic relation between the data fidelity parameter and denoising strength. The second plot shows the corresponding achieved PSNR values (for optimally tuned data fidelity parameters) as a function of $\sigma$. We can see that the PSNR is quite stable over a large range of sufficiently large denoising strengths.}
\label{fig:parameterDependence}
\vspace{-0.27cm}
\end{figure}


To test such an hypothesis we run several different deconvolution experiments with the same input data, but different neural networks which all differ by the standard deviation $\sigma$ they have been trained on. We use a data fidelity term of the form $\frac{\alpha}{2}\|Au-f\|_2^2$ for a blur operator $A$, and data fidelity parameter $\alpha$.  We then run an exhaustive search for the best parameter $\alpha$ maximizing the PSNR value for each of the different neural networks. The first plot of Figure \ref{fig:parameterDependence} illustrates the optimal data fidelity parameter $\alpha$ as a function of the standard deviation $\sigma$ the corresponding neural network has been trained on. Interestingly, the dependence of the optimal $\alpha$ on $\sigma$ indeed seems to be well approximated by a parabola, as illustrated by the dashed blue line representing the curve $\alpha = p~\sigma^2$ for an optimal $p$. 

It is important to note that while in the convex optimization setting a rescaling of both, regularization and data fidelity parameter, does not change the final result at all, the results obtained at each of the data points shown in the first part of Figure~\ref{fig:parameterDependence} do differ as illustrated in the second plot. While a network trained on very small noise did not give good results, a sufficiently large standard deviation gives good results over a large range of training noise level $\sigma$. 

Please also note that similar choices (data fidelity parameter and strength of the denoising algorithm) have to be made for any other custom denoising algorithm: As discussed above, the authors of~\cite{heide14_flexisp} proposed to make the BM3D denoising strength step size depended.~\cite{romano16_red} also considers the use of neural networks as proximal operators, but similar to \cite{heide14_flexisp}, the authors of~\cite{romano16_red} try to make the denoising strength step size dependent. However, since the denoising strength of a neural network cannot be adapted as easily as for the BM3D algorithm, the authors rely on the assumption that a rescaling of the input data which is fed into the network allows to adapt the denoising strength. Instead we propose to rather fix the denoising strength, which -- according to Proposition~\ref{prop:stepsizeInvariance} -- then allows us to fix the algorithm step size $\gamma = 1$ and control the smoothness of the final result by adapting the data fidelity parameter. This avoids the problem of the aforementioned approaches that the internal step size parameter $\gamma$ of the algorithmic scheme influences the result and therefore becomes a (difficult-to-tune) hyperparameter.

\section{Numerical implementation}
\subsection{Algorithmic framework and prior stacking}
In the following section we describe how we implemented the proposed algorithmic scheme with a neural network replacing a proximal operator.

According to Remark~\ref{prop:stationaryPoints} the potential fixed-points of any of the schemes are the same.
In comparison to the PG method, the PDHG algorithm has the advantage that it can easily combine learned (neural network) priors (which have no associated cost function term and thus are referred to as \textit{implicit priors}) with explicitly modeled priors that can be tailored to specific applications -- a fact that has first been exploited by the authors of~\cite{heide14_flexisp} in a technique termed \textit{prior stacking}, which we utilize in our experiments as well.

A combination, or \textit{stacking}, of different priors can easily be achieved in the PDHG algorithm by introducing multiple variables: If we consider all variables in their vectorized form, our final algorithmic scheme is given by 
\begin{align}
    \label{eq:ourPdhg2}
    z^{k+1} =& z^{k} + \gamma D\bar{u}^k - \gamma\text{prox}_{\frac{\beta}{\gamma}J}\left(\frac{1}{\gamma} z^{k} + D\bar{u}^k\right), \\
    y^{k+1} =& y^{k} + \gamma \bar{u}^k - \gamma\mathcal{G}\left(\frac{1}{\gamma} y^{k} + \bar{u}^k\right), \\
    \label{eq:ourPdhg3}
    u^{k+1} =& \text{prox}_{\tau \alpha(H_f \circ A)}(u^k - \tau y^{k+1}- \tau D^Tz^{k+1}) ,\\
    \bar{u}^{k+1}=& 2u^{k+1} -u^{k},
    \end{align}
    where $D$ is an arbitrary linear operator (e.g. the discretized gradient in the case of TV regularization), $J$ an additional regularization (e.g. $J(Du) = \|Du\|_{2,1}$ for the TV), $\beta$ is a regularization parameter, $\alpha$ is the data fidelity parameter, and we use $(H_f\circ A)(u) =\frac{1}{2}\|Au-f\|_2^2$ for a linear operator $A$. We now have two variables $z$ and $y$, which implement the network $\mathcal{G}$ and an additional regularization $J$, where the regularization $J$ may again consist of multiple priors. For more details on prior stacking we refer the reader to \cite{heide14_flexisp}.

Please note that our result of Proposition~\ref{prop:stepsizeInvariance} can easily be extended to the above algorithm, where an arbitrary $\gamma = \frac{c}{\tau}$ can be eliminated via $\beta \rightarrow \frac{\beta}{\gamma}$, $\alpha \rightarrow \frac{\alpha}{\gamma}$, with $c$ (usually) denoting the operator norm $\|[I, -D^T]\|^2$. Consequently, we again only have to optimize for the data fidelity and regularization parameters unless one considers even the product $c=\tau \gamma$ of the step sizes as a free parameter. For the sake of clarity and similarity to the convex optimization case, we decided not to pursue this direction.

    \subsection{Deep convolutional denoising network}
    \label{sec:denoisingNetwork}
    In order to make our denoising network benefit from the recent advances in learning based problem solving we use an end-to-end trained deep convolutional neural network (CNN). Our network architecture of choice is similar to \textit{DnCNN-S}~\cite{zhang16_denoising} and composed of 17 convolution layers with a kernel size of 3$\times$3 each of which is followed by a rectified linear unit (ReLU). 
    Input of the network is either a gray-scale or a color image depending on the application. We use the training pipeline identical to~\cite{zhang16_denoising} with the Adam optimization algorithm \cite{adam} and train our network for removing Gaussian noise of a fixed standard deviation $\sigma$. \refTab{table:evaluation_denoising} demonstrates the superior performance of our learned denoising operator in comparison with general denoising algorithms such as NLM and BM3D on a range of different $\sigma$. It should be noted that each $\sigma$ requires an individually trained \textit{DnCNN-S}. Although we used different noise levels than the one presented in~\cite{zhang16_denoising}, our results have similar margins to BM3D indicating that our trained networks represent state-of-the-art denoising methods. 

    
    \begin{table}

\caption{Average PSNRs in [dB] for 11 test images for different standard deviations $\sigma$ of the Gaussian noise in a comparison of NLM, BM3D, and our denoising networks using the DnCNN-S architecture proposed in~\cite{zhang16_denoising}. We used the same test images as in our deconvolution experiments.}
\vspace{-0.27cm}

\begin{center}
\begin{tabular}{|*{5}{S|}}
\hline
{$\sigma$}     &   {Noisy} &     {NLM~\cite{buades11_nlm}} &    {BM3D~\cite{dabov07_bm3d}} &   {DnCNN-S} \\
\hline
    0.02 & 33.9876 & 35.4869 & 36.7699 &  \textbf{37.80} \\
    0.03 & 30.4691 & 32.7279 & 34.1417 &  \textbf{35.26} \\
    0.04 & 27.9926 & 31.0369 & 32.4908 &  \textbf{33.52}  \\
    0.05 & 26.0554 & 29.7883 & 31.1586 &  \textbf{32.15} \\
    0.06 & 24.4989 & 28.936  & 30.1337 &  \textbf{31.13} \\
    0.07 & 23.1873 & 28.2675 & 29.223  &  \textbf{30.20} \\
    0.08 & 22.0789 & 27.5197 & 28.5722 &  \textbf{29.48} \\
    0.09 & 21.0043 & 26.9394 & 27.8891 &  \textbf{28.81}  \\
    0.1  & 20.1405 & 26.3712 & 27.4078 &  \textbf{28.10} \\
\hline
\end{tabular}
\end{center}

\label{table:evaluation_denoising}
\vspace{-0.27cm}
\end{table}

\section{Evaluation}
\label{sec:evaluation}
The general idea of using neural networks instead of proximal operators applies to any image reconstruction task. We demonstrate the effectiveness of this approach on the exemplary problems of image deconvolution and Bayer demosaicking. It is important to note that we keep the neural network fixed throughout the entire numerical evaluation. In particular, the network has neither been specifically trained for deconvolution nor for demosaicking, but only on removing Gaussian noise with a fixed noise standard deviation of $\sigma_f = 0.02$.

For a direct comparison we follow the experimental setup of~\cite{heide14_flexisp}, but reimplemented the problems using the problem agnostic modeling language for image optimization problems \textit{ProxImaL}~\cite{heide16_proximal}. For the denoising network we used the graph computation framework \textit{TensorFlow}~\cite{tensorflow2015_whitepaper} which made the integration simple and flexible.~\footnote{Our code is available at~\url{https://github.com/tum-vision/learn_prox_ops}.} Since our approach stands in direct comparison to~\cite{heide14_flexisp}, we have to mention that we were not able to reproduce their results with our implementation. This is likely due to them replacing the proximal operator with an improved but not released version of BM3D which was even further refined for the case of demosaicking. In this paper, our main goal is to compare our approach with the framework of \cite{heide14_flexisp} as methods that are not tailored to a specific problem but provide solutions for any linear inverse problem. Therefore, we use the publicly available BM3D implementation, perform a grid search over all free parameters, and denote the obtained results in our evaluation by \textit{FlexISP$^*$}. The latter allows us to investigate to what extend the advantage in denoising performance shown in~\refTab{table:evaluation_denoising} transfers to general inverse problems. Of course, approaches that are tailored to a specific problem, e.g. by training a specialized network, will likely yield superior performance. 

\textit{FlexISP$^*$} applies the same step size related denoising approach as \cite{heide14_flexisp}, but in contrast to ~\cite{heide14_flexisp} we observed a notable effect of the choice of $\gamma$ and therefore included it in the parameter optimization. We set the same residual-based stopping criterion as well as a maximum number of 30 PDHG iterations for \textit{FlexISP$^*$} and our approach. 

    \subsection{Demosaicking}
    \begin{figure*}[t]
\contourlength{1.0pt}
\setlength{\unitlength}{0.16\textwidth}
\centering

\setlength{\fboxsep}{0pt}%
\setlength{\fboxrule}{0.01\unitlength}%

\newcommand{\img}[1]{
    \ifarxiv
        \includegraphics[width=\unitlength]{image_comparison_demosaicking/#1_small}
    \else
        \includegraphics[width=\unitlength]{image_comparison_demosaicking/#1_small}
    \fi
}

\newcommand{\imgdiff}[1]{
    \ifarxiv
        \fcolorbox{white}{white}{\includegraphics[trim={200 175 200 225},clip,width=0.4\unitlength]{image_comparison_demosaicking/#1_small}}
    \else
        \fcolorbox{white}{white}{\includegraphics[trim={200 175 200 225},clip,width=0.4\unitlength]{image_comparison_demosaicking/#1_small}}
    \fi
}

\newcommand{\imgdifff}[1]{
    \ifarxiv
        \fcolorbox{white}{white}{\includegraphics[trim={350 250 50 150},clip,width=0.4\unitlength]{image_comparison_demosaicking/#1_small}}
    \else
        \fcolorbox{white}{white}{\includegraphics[trim={350 250 50 150},clip,width=0.4\unitlength]{image_comparison_demosaicking/#1_small}}
    \fi
}

\subfloat{
\begin{picture}(1,1) 
\put(0,0){\img{1_ground_truth}} 
\put(0.75,0.89){\makebox(0,0)[B]{\footnotesize \contour{black}{\textcolor{white}{Original}}}} 
\end{picture}

\begin{picture}(1,1) 
\put(0,0){\img{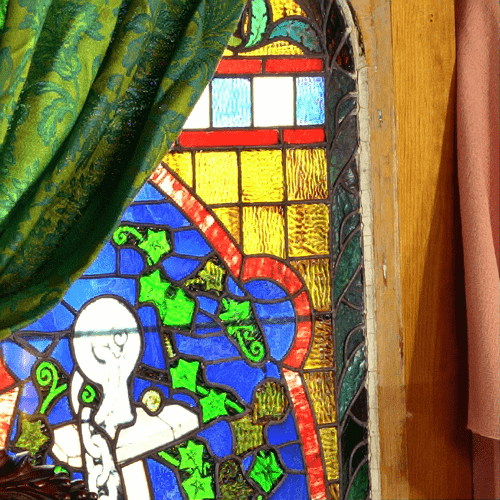}} 
\put(0.05,0.54){\imgdiff{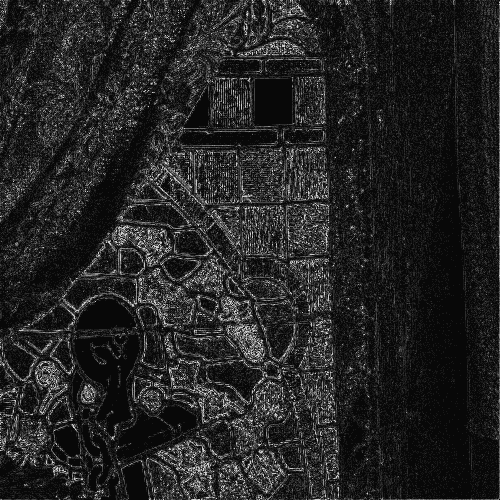}}
\put(0.5,0.1){\makebox(0,0)[B]{\footnotesize \contour{black}{\textcolor{white}{28.49 dB}}}} 
\put(0.75,0.89){\makebox(0,0)[B]{\footnotesize \contour{black}{\textcolor{white}{DCRAW}}}} 
\end{picture}

\begin{picture}(1,1) 
\put(0,0){\img{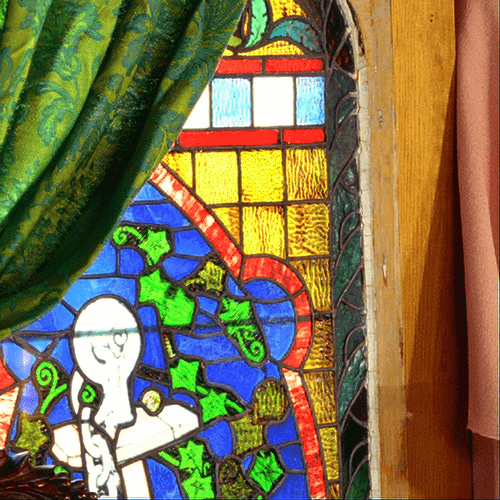}} 
\put(0.05,0.54){\imgdiff{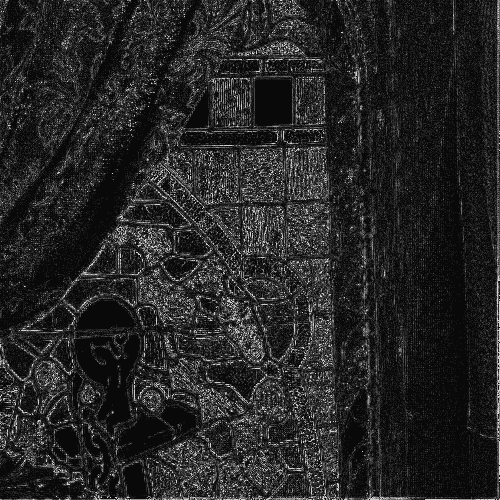}}
\put(0.5,0.1){\makebox(0,0)[B]{\footnotesize \contour{black}{\textcolor{white}{28.18 dB}}}} 
\put(0.75,0.89){\makebox(0,0)[B]{\footnotesize \contour{black}{\textcolor{white}{ADOBE}}}} 
\end{picture}

\begin{picture}(1,1) 
\put(0,0){\img{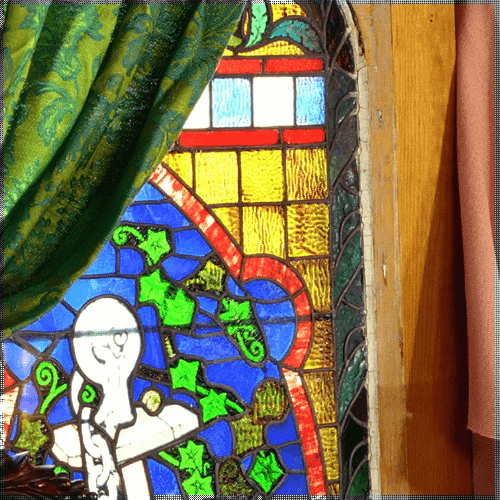}} 
\put(0.05,0.54){\imgdiff{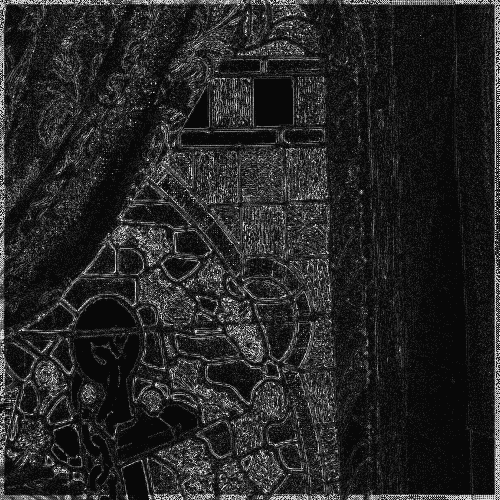}}
\put(0.5,0.1){\makebox(0,0)[B]{\footnotesize \contour{black}{\textcolor{white}{29.18 dB}}}} 
\put(0.75,0.89){\makebox(0,0)[B]{\footnotesize \contour{black}{\textcolor{white}{LDI-NAT}}}} 
\end{picture}

\begin{picture}(1,1) 
\put(0,0){\img{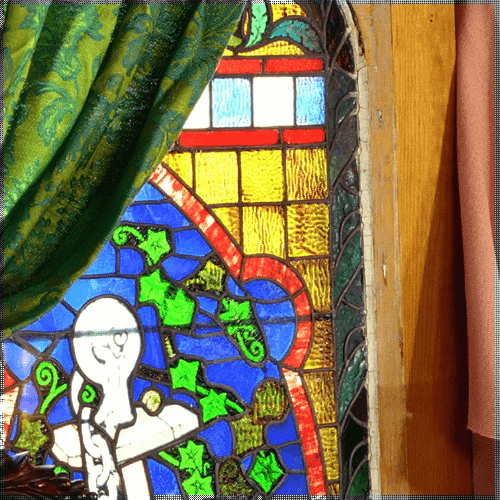}} 
\put(0.05,0.54){\imgdiff{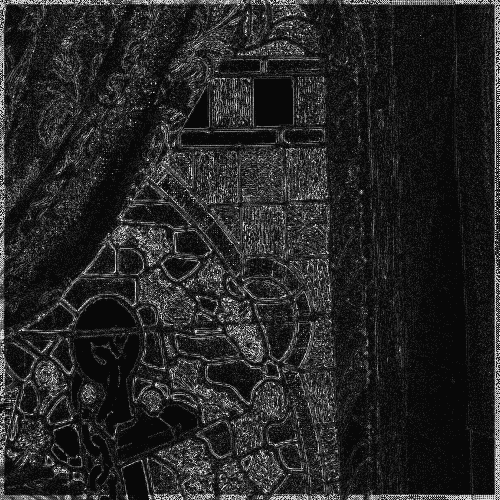}}
\put(0.5,0.1){\makebox(0,0)[B]{\footnotesize \contour{black}{\textcolor{white}{28.93 dB}}}} 
\put(0.75,0.89){\makebox(0,0)[B]{\footnotesize \contour{black}{\textcolor{white}{FlexISP*}}}} 
\end{picture}

\begin{picture}(1,1) 
\put(0,0){\img{1_ours}}
\put(0.05,0.54){\imgdiff{1_ours_diff}}
\put(0.5,0.1){\makebox(0,0)[B]{\footnotesize \contour{black}{\textcolor{white}{29.38 dB}}}} 
\put(0.75,0.89){\makebox(0,0)[B]{\footnotesize \contour{black}{\textcolor{white}{Ours}}}} 
\end{picture}

} 

\subfloat{
\begin{picture}(1,1)
\put(0,0){\img{5_ground_truth}}
\end{picture}

\begin{picture}(1,1)
\put(0,0){\img{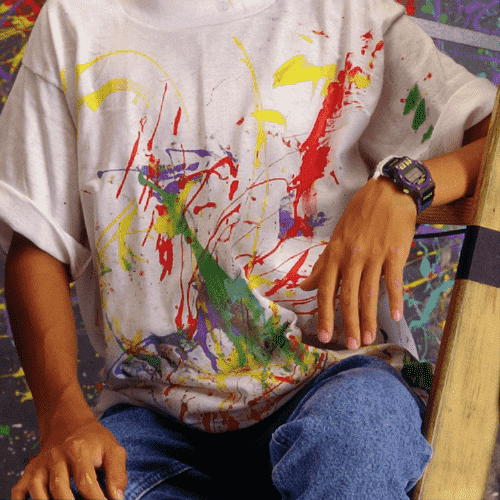}}
\put(0.05,0.54){\imgdifff{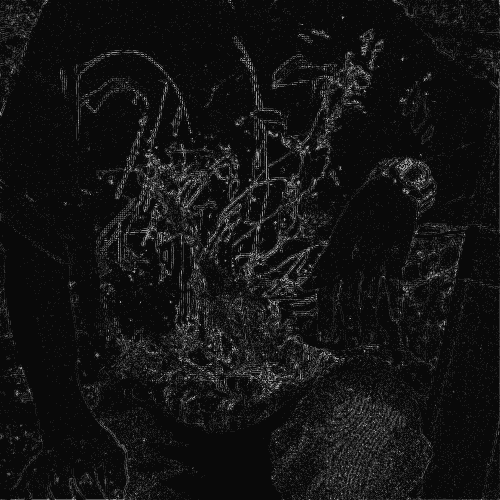}}
\put(0.5,0.1){\makebox(0,0)[B]{\footnotesize \contour{black}{\textcolor{white}{33.65 dB}}}} 
\end{picture}

\begin{picture}(1,1)
\put(0,0){\img{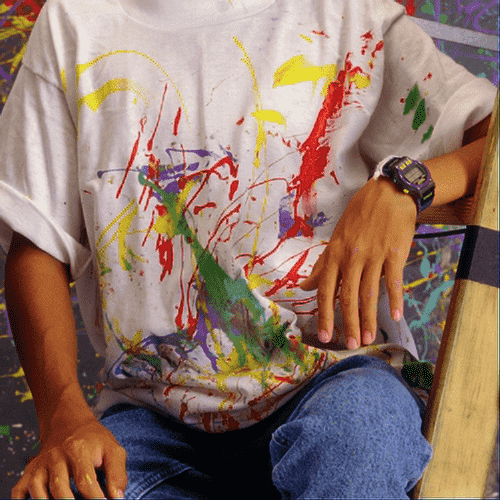}}
\put(0.05,0.54){\imgdifff{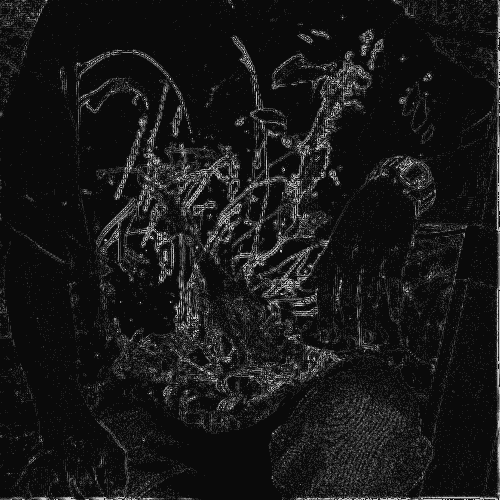}}
\put(0.5,0.1){\makebox(0,0)[B]{\footnotesize \contour{black}{\textcolor{white}{31.98 dB}}}} 
\end{picture}

\begin{picture}(1,1)
\put(0,0){\img{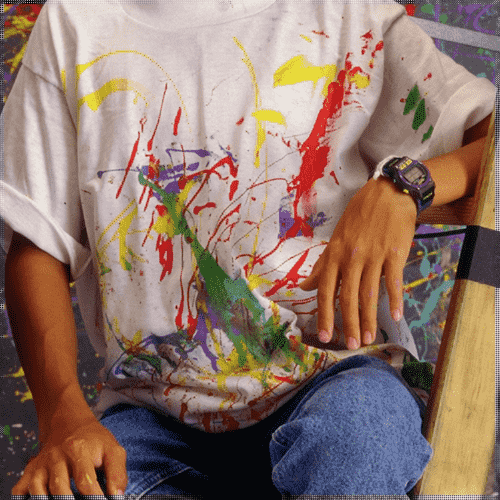}}
\put(0.05,0.54){\imgdifff{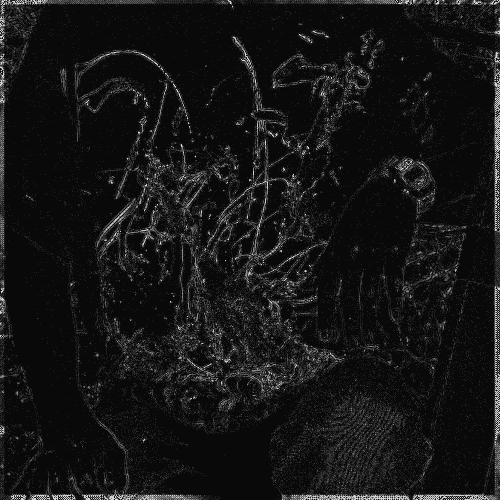}}
\put(0.5,0.1){\makebox(0,0)[B]{\footnotesize \contour{black}{\textcolor{white}{33.92 dB}}}} 
\end{picture}

\begin{picture}(1,1)
\put(0,0){\img{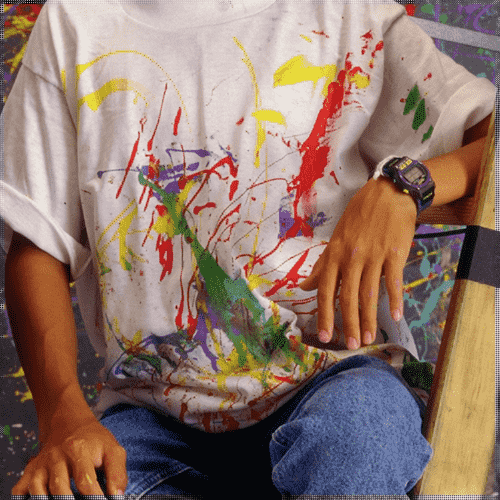}}
\put(0.05,0.54){\imgdifff{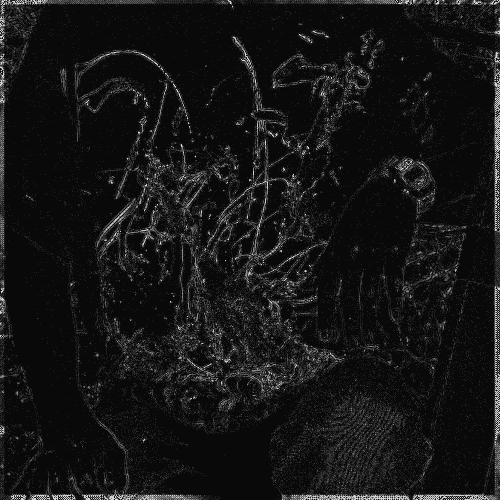}}
\put(0.5,0.1){\makebox(0,0)[B]{\footnotesize \contour{black}{\textcolor{white}{34.22 dB}}}} 
\end{picture}

\begin{picture}(1,1)
\put(0,0){\img{5_ours}}
\put(0.05,0.54){\imgdifff{5_ours_diff}}
\put(0.5,0.1){\makebox(0,0)[B]{\footnotesize \contour{black}{\textcolor{white}{34.89 dB}}}} 
\end{picture}

}

\caption{Visual comparison of different demosaicking methods on two example images of the McMaster color image data set. To illustrate the differences in reconstruction quality we added zoomed in residual images. Apart from \textit{FlexISP}$^*$ and our result, all other images are taken from \cite{heide14_flexisp}.}
\label{fig:image_comparison_demosaicking}
\vspace{-0.27cm}
\end{figure*}
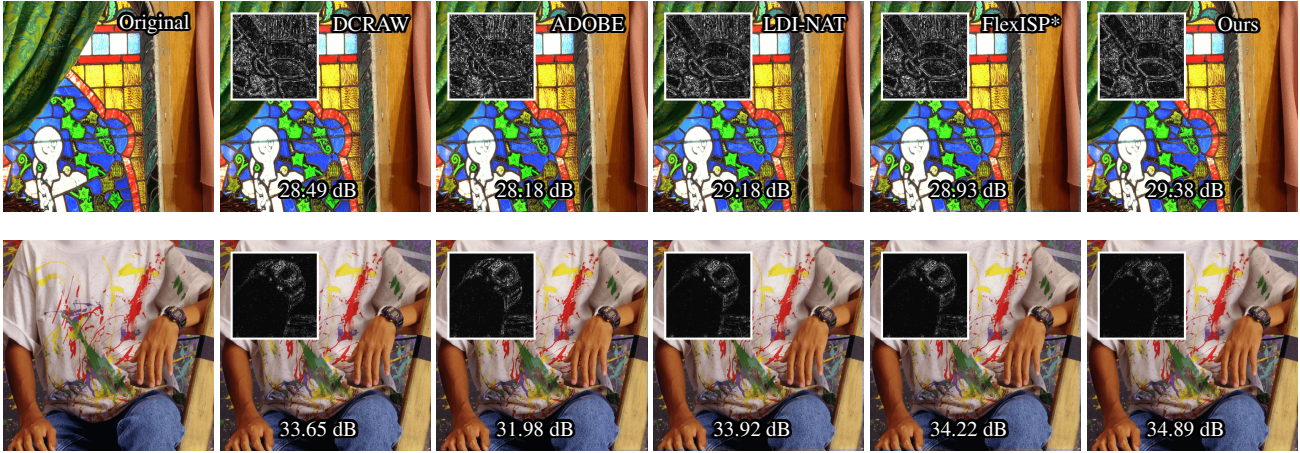

    We evaluated our performance on noise-free demosaicking of the Bayer filtered \textit{McMaster} color image dataset,~\cite{zhang2011_mcmaster}. Besides our denoising network, we use the cross-channel and total variation prior as additional explicit regularizations $J$ in~\refEq{eq:ourPdhg2} as also done in~\cite{heide14_flexisp}. For \textit{FlexISP$^*$} as well as for our method we optimized in an exhaustive grid search for the data fidelity parameter $\alpha$ as well as for the regularization parameters $\beta_{TV}$ and $\beta_{Cross}$.

    \begin{figure}[t]
\begin{center}
  \includegraphics[width=\linewidth]{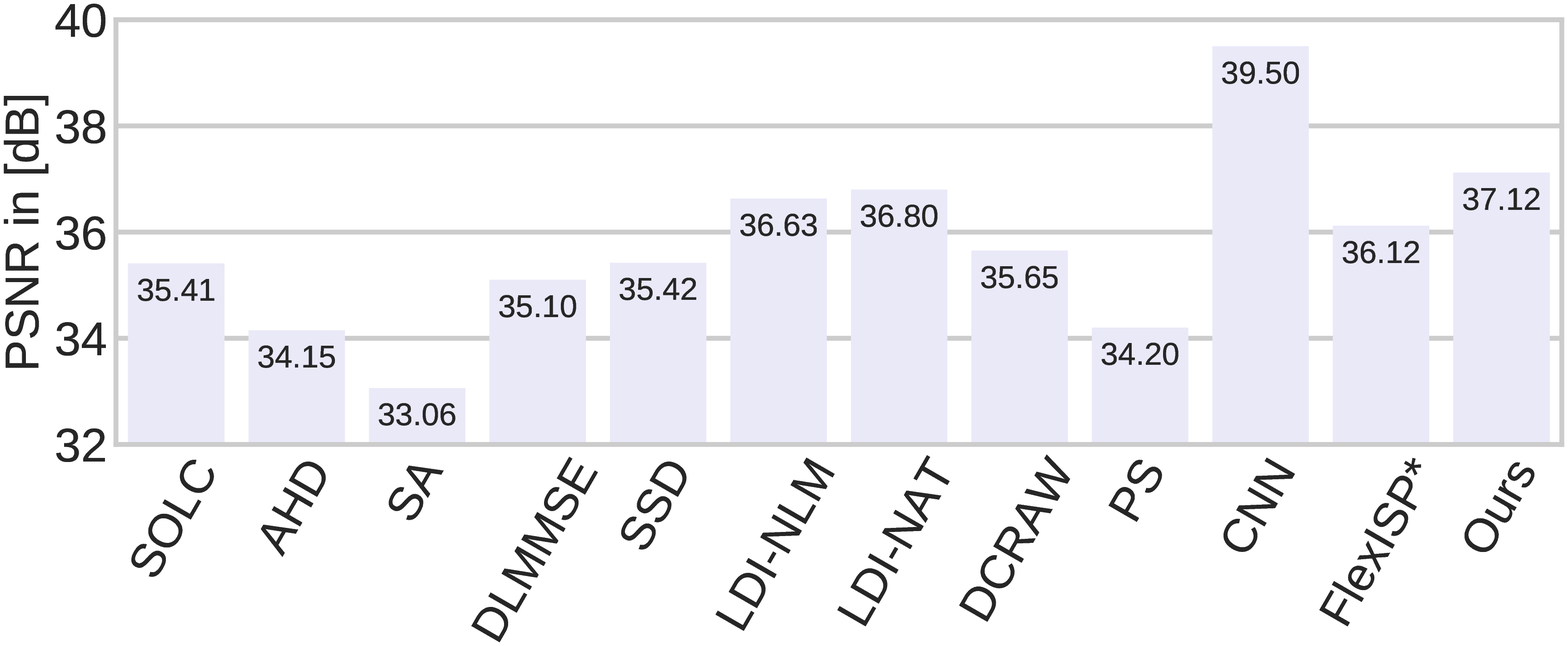}
\end{center}
\caption{Average PSNR results in [dB] for demosaicking the McMaster color image dataset. The results of all methods except CNN, \textit{FlexISP$^*$} and ours are copied from the same comparison in \cite{heide14_flexisp}. As expected the deep CNN from \cite{gharbi16_demosaicking} which was specifically trained on demosaicking outperforms our approach. Nevertheless the results show that using a powerful denoising network as a proximal operator yields substantial results.}
\label{fig:evaluation_demosaicking}
\vspace{-0.27cm}
\end{figure}
    
    \refFig{fig:evaluation_demosaicking} compares our average debayering quality with multiple state-of-the-art algorithms, and \refFig{fig:image_comparison_demosaicking} gives a visual impression of the demosaicking quality of the corresponding algorithms for two example images. 
    As we can see, 
    the proposed method achieves a very high average PSNR value and is only surpassed by \cite{gharbi16_demosaicking} who specifically trained a deep demosaicking CNN. Comparing our approach with \textit{FlexISP$^*$}, the advantage of about $1$dB in PSNR values of our network over BM3D on image denoising carried over to the inverse problem of demosaicking. 

    To justify our choice of a fixed $\sigma_f$ we investigate the robustness of our approach to different choices of denoising networks. \refTab{table:evaluation_demosaicking_params} illustrates the results of our method for differently trained networks, and also shows the optimal parameters found by our grid search. While we can see that the PSNRs do vary by about $1.1$dB, it is encouraging to see that the average PSNR remains above $36$dB for a wide range of differently trained networks. A little less conclusive are the optimal parameters found by our grid search. They merely seem to indicate that explicit priors should be used less if the denoising network is trained on larger noise levels. We also tested completely omitting explicit priors, which decreased the average performance by about $0.4$dB.
    


    \begin{table}[t]

\caption{The table shows the optimal parameters for the data fidelity parameter $\alpha$, the TV regularization $\beta_{TV}$ and the cross channel prior $\beta_{Cross}$ when denoising networks trained on Gaussian noise with different standard deviation $\sigma$ are used. Below the parameters we show the average PSNR values in [dB] obtained on the McMaster color image data set. Considering the results of competing methods shown in \refFig{fig:evaluation_demosaicking}, different denoising networks yield quite good demosaicking performance on a wide range of different $\sigma$.}
\vspace{-0.27cm}
\footnotesize
\begin{center}
\begin{tabular}{|c *{3}{|c} |}
\hline
\multirow{2}{*}{$\sigma$}       & \multicolumn{3}{c|}{Reconstruction PSNR in [\textit{dB}]} \\
                                \cline{2-4}
                                & {$\alpha$} & {$\beta_{TV}$} & {$\beta_{Cross}$} \\

\hline
\hline
\multirow{2}{*}{$0.001$}          & \multicolumn{3}{S|}{36.050674}\\
                                  \cline{2-4}
                                  & 4000 & 0.1 & 0.05 \\
\hline
\hline
\multirow{2}{*}{$0.01$}           & \multicolumn{3}{S|}{36.740009}\\
                                  \cline{2-4}
                                  & 100 & 0.01 & 0.0 \\
\hline
\hline
\multirow{2}{*}{$0.02$}           & \multicolumn{3}{S|}{37.123543}\\
                                  \cline{2-4}
                                  & 90 & 0.01 & 0.0 \\
\hline
\hline
\multirow{2}{*}{$0.03$}           & \multicolumn{3}{S|}{36.394198}\\
                                  \cline{2-4}
                                  & 12 & 0.0 & 0.0 \\
\hline
\hline
\multirow{2}{*}{$0.05$}           & \multicolumn{3}{S|}{36.082635}\\
                                  \cline{2-4}
                                  & 800 & 0.0 & 0.01 \\
\hline
\end{tabular}
\end{center}

\label{table:evaluation_demosaicking_params}
\vspace{-0.27cm}
\end{table}

    \subsection{Deconvolution}
    \begin{figure*}[t]
    \setlength{\unitlength}{0.119\textwidth}
    \centering


\subfloat{
\begin{picture}(1,1) 
\put(0,0){\includegraphics[width=\unitlength]{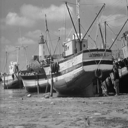}} 
\put(0.5,0.89){\makebox(0,0)[B]{\footnotesize \contour{black}{\textcolor{white}{Original}}}} 
\end{picture}

\begin{picture}(1,1) 
\put(0,0){\includegraphics[width=\unitlength]{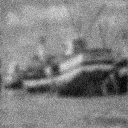}} 
\put(0.5,0.1){\makebox(0,0)[B]{\footnotesize \contour{black}{\textcolor{white}{21.95 dB}}}} 
\put(0.5,0.89){\makebox(0,0)[B]{\footnotesize \contour{black}{\textcolor{white}{Blurred}}}} 
\end{picture}

\begin{picture}(1,1) 
\put(0,0){\includegraphics[width=\unitlength]{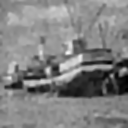}} 
\put(0.5,0.1){\makebox(0,0)[B]{\footnotesize \contour{black}{\textcolor{white}{24.29 dB}}}} 
\put(0.5,0.89){\makebox(0,0)[B]{\footnotesize \contour{black}{\textcolor{white}{IRLS}}}} 
\end{picture}

\begin{picture}(1,1) 
\put(0,0){\includegraphics[width=\unitlength]{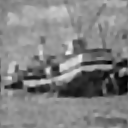}} 
\put(0.5,0.1){\makebox(0,0)[B]{\footnotesize \contour{black}{\textcolor{white}{24.32 dB}}}} 
\put(0.5,0.89){\makebox(0,0)[B]{\footnotesize \contour{black}{\textcolor{white}{LUT}}}} 
\end{picture}

\begin{picture}(1,1) 
\put(0,0){\includegraphics[width=\unitlength]{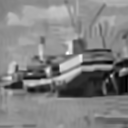}} 
\put(0.5,0.1){\makebox(0,0)[B]{\footnotesize \contour{black}{\textcolor{white}{24.47 dB}}}} 
\put(0.5,0.89){\makebox(0,0)[B]{\footnotesize \contour{black}{\textcolor{white}{IDD-BM3D}}}} 
\end{picture}

\begin{picture}(1,1) 
\put(0,0){\includegraphics[width=\unitlength]{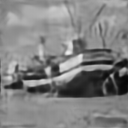}} 
\put(0.5,0.1){\makebox(0,0)[B]{\footnotesize \contour{black}{\textcolor{white}{24.60 dB}}}} 
\put(0.5,0.89){\makebox(0,0)[B]{\footnotesize \contour{black}{\textcolor{white}{MLP}}}} 
\end{picture}

\begin{picture}(1,1) 
\put(0,0){\includegraphics[width=\unitlength]{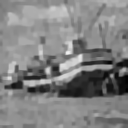}} 
\put(0.5,0.1){\makebox(0,0)[B]{\footnotesize \contour{black}{\textcolor{white}{24.44 dB}}}} 
\put(0.5,0.89){\makebox(0,0)[B]{\footnotesize \contour{black}{\textcolor{white}{FlexISP*}}}} 
\end{picture}

\begin{picture}(1,1) 
\put(0,0){\includegraphics[width=\unitlength]{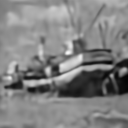}} 
\put(0.5,0.1){\makebox(0,0)[B]{\footnotesize \contour{black}{\textcolor{white}{24.41 dB}}}} 
\put(0.5,0.89){\makebox(0,0)[B]{\footnotesize \contour{black}{\textcolor{white}{Ours}}}} 
\end{picture}

} 

\subfloat{
\includegraphics[width=\unitlength]{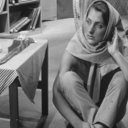} 

\begin{picture}(1,1)
\put(0,0){\includegraphics[width=\unitlength]{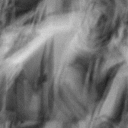}}
\put(0.5,0.1){\makebox(0,0)[B]{\footnotesize \contour{black}{\textcolor{white}{17.56 dB}}}} 
\end{picture}

\begin{picture}(1,1)
\put(0,0){\includegraphics[width=\unitlength]{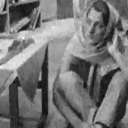}}
\put(0.5,0.1){\makebox(0,0)[B]{\footnotesize \contour{black}{\textcolor{white}{29.73 dB}}}} 
\end{picture}

\begin{picture}(1,1)
\put(0,0){\includegraphics[width=\unitlength]{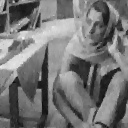}}
\put(0.5,0.1){\makebox(0,0)[B]{\footnotesize \contour{black}{\textcolor{white}{29.15 dB}}}} 
\end{picture}

\begin{picture}(1,1)
\put(0,0){\includegraphics[width=\unitlength]{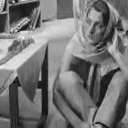}}
\put(0.5,0.1){\makebox(0,0)[B]{\footnotesize \contour{black}{\textcolor{white}{30.69 dB}}}} 
\end{picture}

\begin{picture}(1,1)
\put(0,0){\includegraphics[width=\unitlength]{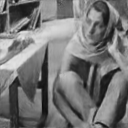}}
\put(0.5,0.1){\makebox(0,0)[B]{\footnotesize \contour{black}{\textcolor{white}{30.53 dB}}}} 
\end{picture}

\begin{picture}(1,1)
\put(0,0){\includegraphics[width=\unitlength]{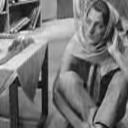}}
\put(0.5,0.1){\makebox(0,0)[B]{\footnotesize \contour{black}{\textcolor{white}{30.59 dB}}}} 
\end{picture}

\begin{picture}(1,1)
\put(0,0){\includegraphics[width=\unitlength]{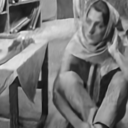}}
\put(0.5,0.1){\makebox(0,0)[B]{\footnotesize \contour{black}{\textcolor{white}{31.67 dB}}}} 
\end{picture}

} 
     
\caption{Visual comparison of different deconvolution methods on two out of 11 standard test images. The images \textit{Boat} and \textit{Barbara} were each corrupted with Gaussian noise ($\sigma = 0.04$, $\sigma = 0.01$) and a Gaussian blur (experiment \textit{a}) as well as a motion blur (experiment \textit{e}), respectively. Apart from \textit{FlexISP}$^*$ and our result, all other images are taken from \cite{heide14_flexisp}.}
\label{fig:image_comparison_deblurring}
\vspace{-0.27cm}
\end{figure*}

    For evaluating the deconvolution performance, we use the benchmark introduced by~\cite{schuler13_mlp}, which consists of 5 different experiments with different Gaussian noise and different blur kernels applied to 11 standard test images. Experiments \textit{a} - \textit{c}, \textit{d} and \textit{e} each apply a Gaussian, squared and motion blurring, respectively.
 
    \refTab{table:evaluation_deblurring} compares our average results over all test images with eight state-of-the-art deblurring methods, and \refFig{fig:image_comparison_deblurring} gives a visual impression of the corresponding results for two example images within experiments \textit{a} and \textit{e}. Apart from \textit{FlexISP$^*$} and our method, all other results are taken from \cite{heide14_flexisp}. For \textit{FlexISP$^*$} and our method, we used the TV as an explicit additional prior and optimized individual parameter sets for each experiment. 
    However, while \textit{FlexISP$^*$} benefits from a separately optimized stepsize $\gamma$, our method applies the same neural network for all experiments.
    Nevertheless, our overall performance is on par with the other methods. 
    
    Particularly remarkable is the fact that the MLP approach form~\cite{schuler13_mlp} trained a network (including the different linear operators) on each of the five experiments separately. It is encouraging to see that an energy minimization algorithm with a generic denoising network as a proximal operator yields results similar to the specialized networks in experiments \textit{a} - \textit{d} and even outperformed the latter on the problem \textit{e} of removing motion blurs. 

    \begin{table}[t]

\caption{Average PSNR results in [dB] for image deconvolution on a set of 11 standard grayscale images over 5 experiments with different blur kernels and noise levels as detailed in \cite{schuler13_mlp}. All reported values except \textit{FlexISP}$^*$ and ours were taken from \cite{heide14_flexisp}. Note that we used exactly the same denoising network ($\sigma= 0.02$) for all experiments opposed to MLP, which trained specialized neural networks removing the different corruptions of experiments \textit{a}--\textit{e} separately. We conclude that only very little performance has to be scarified when combining a generic but powerful denoising network with the flexibility of energy minimization algorithms.}
\vspace{-0.27cm}
\begin{center}
\resizebox{\linewidth}{!}{%

\begin{tabular}{| c | *{6}{S |}}
\hline
\multirow{2}{*}{Deblurring method}      & \multicolumn{6}{c|}{Reconstruction PSNR in [\textit{dB}]} \\
                                          \cline{2-7}
                                        & {\textit{a}}    & {\textit{b}}    & {\textit{c}}    & {\textit{d}}    & {\textit{e}}    & {AVG} \\

\hline
\hline

EPLL~\cite{zoran11_epll}               & 24.04              & 26.64             & 21.36             & 21.04             & 29.25             & 24.47 \\
IRLS ~\cite{levin07}                   & 24.09              & 26.51             & 21.72             & 21.91             & 28.33             & 24.51 \\
LUT ~\cite{krishnan09}                 & 24.17              & 26.60             & 21.73             & 22.07             & 28.17             & 24.55 \\
DEB-BM3D~\cite{dabov08_deb_bm3d}       & 24.19              & 26.30             & 21.48             & 22.20             & 28.26             & 24.49 \\
IDD-BM3D~\cite{danielyan12_idd_bm3d}   & 24.68              & 27.13             & 21.99             & 22.69             & 29.41             & 25.18 \\
FoE~\cite{roth09_foe}                  & 24.07              & 26.56             & 21.61             & 22.04             & 28.83             & 24.62 \\
MLP~\cite{schuler13_mlp}               & \textbf{24.76}     & \textbf{27.23}    & \textbf{22.20}    & \textbf{22.75}    & 29.42             & \textbf{25.27} \\
FlexISP$^*$~\cite{heide14_flexisp}     & 24.32491818        & 26.84293636       & 21.99175455       & 22.52678182       & 29.29994545       & 24.997267 \\
Ours & 24.51              & 27.08             & 21.83             & 21.96             & \textbf{30.17}             & 25.11 \\
\hline
\hline
Ours,{\setlength{\thickmuskip}{0mu} $\sigma= 0.01$}
 &24.2516454545455 &  27.0142363636364 &   21.5652909090909  &  21.5201363636364  &  28.7750909090909 & 24.6253 \\

Ours,{\setlength{\thickmuskip}{0mu} $\sigma= 0.04$}
    & 24.5639090909091  & 27.1025727272727 &   21.9527727272727 &   22.4001090909091  &  30.3493545454545             & 25.2737 \\
Ours,{\setlength{\thickmuskip}{0mu} $\sigma= 0.06$}
    & 24.6150363636364  & 27.1418272727273 &     22.0296454545455 &   22.5804272727273  &  30.2616545454545           & 25.3257 \\
Ours,{\setlength{\thickmuskip}{0mu} $\sigma= 0.09$}
    & 24.5745636363636  &  27.1319363636364  &  21.9779181818182 &   22.6046181818182 &    29.9638181818182             & 25.2506 \\
Ours,{\setlength{\thickmuskip}{0mu} $\sigma= 0.2$}
    & 24.4809090909091 & 26.6250363636364   & 21.9973545454545 &   22.3517454545455 &   25.7952818181818             & 24.2501 \\

\hline
\end{tabular}

}

\end{center}

\label{table:evaluation_deblurring}
\vspace{-0.27cm}
\end{table}

When comparing to the \textit{FlexISP$^*$} results it is interesting to see that the performance advantage our denoising networks have over BM3D on plain denoising did not fully carry over to the deconvolution problem, yielding a comparably small difference in PSNR value. Therefore, a detailed understanding for which problems and in what sense the performance of a denoising algorithm can be fully transferred to an inverse problem when the algorithm is used as a proximal operator remains an open question for future research. 

Due to the efficiency of the neural network, the average runtime of our approach for image deconvolution was ${\approx}2.5s$ in comparison to ${\approx}4s$ of \textit{FlexISP$^*$} yielding a significant relative improvement of $37.5\%$. In both cases the denoising operator was evaluated on the GPU.

We again study the robustness of the proposed approach to networks trained on different noise levels. The second plot of ~\refTab{table:evaluation_deblurring} shows the optimal PSNR values attained with networks that have been trained on different standard deviations $\sigma$.  As we can see the PSNRs remain very stable over a large range of different $\sigma$ indicating the robustness toward the specific network that is used.




  
  

\section{Conclusion}
    In this paper we studied the use of denoising neural networks as proximal operators in energy minimization algorithms. We showed that four different algorithms using neural networks as proximal operators have the same potential fixed-points. Moreover, the particular choice of step size in the PDHG algorithm merely rescales the data fidelity (and other possible regularization) parameters. Interestingly, the noise level the neural network is trained on behaves very much like a regularization parameter derived from MAP estimates and reveals a quadratic relation between the standard deviation $\sigma$ and the data fidelity parameter. 

    For our numerical experiments we proposed to combine the PDHG algorithm with a DnCNN-S denoising network \cite{zhang16_denoising} as a proximal operator and the prior stacking approach of \cite{heide14_flexisp}. 
    Our reconstruction results and robustness tests on the exemplary problems of demosaicking and deblurring indicate that one can obtain state-of-the-art results with a fixed neural network.

    We expect that this concept can significantly ease the need for problem-specific retraining of classical deep learning approaches and additionally even allows to benefit from learned natural image priors for problems where training data is not available. 

\small
\PAR{Acknowledgements.} M.M. and D.C. acknowledge the support of the German Research Foundation (DFG) via the research training group GRK 1564 Imaging New Modalities and the ERC Consolidator Grant ``3D-Reloaded'', respectively.


\clearpage

{\small
\bibliographystyle{ieee}
\bibliography{deep_variationals}
}

\ifarxiv
    \pagenumbering{gobble}
    \setcounter{section}{0}
    \setcounter{figure}{0}
    \setcounter{table}{0}
    \setcounter{equation}{0}
\title{Learning Proximal Operators: \\ Using Denoising Networks for Regularizing Inverse Imaging Problems \\ {\normalfont Supplementary Material}}

\author{Tim Meinhardt$^1$\\
{\tt\scriptsize tim.meinhardt@tum.de}
\and Michael Moeller$^2$\\
{\tt\scriptsize michael.moeller@uni-siegen.de}
\and Caner Hazirbas$^1$\\
{\tt\scriptsize hazirbas@cs.tum.edu}
\and Daniel Cremers$^1$\\
{\tt\scriptsize cremers@tum.de} \\ \and \vspace{0.2cm}
\normalfont{Technical University of Munich$^1$ \quad University of Siegen$^2$} 
}

\ifarxiv
    \date{}
\fi

\maketitle
\thispagestyle{empty}

\begin{abstract}
The supplementary material contains the proof of Remark~3.1 as well as some additional information about the numerical experiments that contribute to the understanding of the main paper. We present detailed qualitative and quantitative evaluation results for each of our two (demosaicking and deconvolution) exemplary linear inverse image reconstruction problems. These results include parameter values obtained with our grid search, reconstruction PSNR values and images.
\end{abstract}

\section*{Proof of Remark~3.1}
For the sake of readability let us restate the remark and the four algorithms with the proximal operators of the regularization $R$ replaced by an arbitrary continuous function $\mathcal{G}$.

\begin{myequ}{PG}
\begin{align}
    \label{eq_supp:proxGradNeuralNet}
    u^{k+1} = \mathcal{G}\left(u^k - \tau A^* \nabla H_f(Au^k) \right).
    \end{align}
\end{myequ}

\begin{myequ}{ADMM}
 \begin{align}
    \label{eq_supp:admm1}
    u^{k+1} =& \text{prox}_{\frac{1 }{\gamma} (H_f \circ A)}\left(v^{k+1}-\frac{1 }{\gamma}y^{k}\right), \\
            \label{eq_supp:admm2}
    v^{k+1} =& \mathcal{G}\left(u^{k}+\frac{1 }{\gamma}y^k\right), \\
    \label{eq_supp:admm3}
    y^{k+1} =& y^k + \gamma(u^{k+1} - v^{k+1}),
    \end{align}
\end{myequ}
\begin{myequ}{PDHG1}
\begin{align}
    \label{eq_supp:pdhg1}
    z^{k+1} =& z^{k} + \gamma A\bar{u}^k - \gamma \text{prox}_{\frac{1}{\gamma}H_f}\left(\frac{1}{\gamma} z^{k} + A\bar{u}^k\right), \\
    \label{eq_supp:pdhg2}
    y^{k+1} =& y^{k} + \gamma \bar{u}^k - \gamma \mathcal{G}\left(\frac{1}{\gamma} y^{k} + \bar{u}^k\right), \\
    \label{eq_supp:pdhg3}
    u^{k+1} =& u^k - \tau A^T z^{k+1} - \tau y^{k+1} ,\\
        \label{eq_supp:pdhg4}
    \bar{u}^{k+1}=& u^{k+1} + \theta(u^{k+1} -u^{k}),
    \end{align}
    \end{myequ}
\begin{myequ}{PDHG2}
  \begin{align}
    \label{eq_supp:pdhg2b}
    y^{k+1} =& y^{k} + \gamma \bar{u}^k - \gamma \mathcal{G}\left(\frac{1}{\gamma} y^{k} + \bar{u}^k\right), \\
    \label{eq_supp:pdhg3b}
    u^{k+1} =& \text{prox}_{\tau (H_f \circ A)}(u^k - \tau y^{k+1}) ,\\
    \label{eq_supp:pdhg4b}
    \bar{u}^{k+1}=& u^{k+1} + \theta(u^{k+1} -u^{k}).
    \end{align}
     \end{myequ}
 \begin{remark}[Remark 3.1 in main Paper]
    \label{prop_supp:stationaryPoints}
    Consider replacing the proximal operator of $R$ in the PG, ADMM, PDHG1, and PDHG2 methods by an arbitrary continuous function $\mathcal{G}$. Then the fixed-point equations of all four resulting algorithmic schemes are equivalent, and yield  
%
    \begin{align}
    \label{eq_supp:fixedPoint}
     u_* = \mathcal{G}\left(u_* - t A^T\nabla H_f (Au_*)\right)
    \end{align}
    with $* \in \{ \text{PG}, \text{ADMM}, \text{PDHG1}, \text{PDHG2}\}$ and $t=\tau$ for PG and PDHG2, and $t=\frac{1}{\gamma}$ for ADMM and PDHG1. 
\end{remark}
\begin{proof}

For the PG-based algorithmic scheme the statement follows immediately as \eqref{eq_supp:fixedPoint} coincides with the update equation \eqref{eq_supp:proxGradNeuralNet}. 

At fixed-points of the ADMM-based scheme, it follows from~\refEq{eq_supp:admm3} that $u_{ADMM} = v$. The optimality condition for~\refEq{eq_supp:admm1} therefore becomes $y = -A^T\nabla H_f(Au_{ADMM})$, such that~\refEq{eq_supp:admm2} shows the fixed-point~\refEq{eq_supp:fixedPoint} for the ADMM-based scheme. Vice versa, for any given element $u^0$ meeting~\refEq{eq_supp:fixedPoint} one initializes $y^0 = -A^T\nabla H_f(Au^0)$, and $v^0 = u^0$ to obtain a fixed-point of the ADMM-based scheme. 

At fixed-points of the PDHG1-based scheme (again variables without superscripts denoting the fixed-point), it follows from~\refEq{eq_supp:pdhg3} that $y = - A^Tz$. The optimality condition for~\refEq{eq_supp:pdhg1} yields
    \begin{align}
    & 0=Au - \frac{1}{\gamma}z - Au + \frac{1}{\gamma}\nabla H_f(Au), \\
    \Rightarrow \qquad & z=\nabla H_f(Au) ,
    \end{align}
    and inserting the resulting identity $y = -A^T\nabla H_f(Au)$ into~\refEq{eq_supp:pdhg2} shows that any fixed-point of the PDHG1-based scheme meets~\refEq{eq_supp:fixedPoint}. For a given fixed-point $u^0$ meeting~\refEq{eq_supp:fixedPoint} the choices $\bar{u}^0=u^0$, $z^0 = \nabla H_f(Au^0)$, $y^0 = -A^T \nabla H_f(Au^0)$ yield a fixed-point of the PDHG1-based algorithmic scheme.

Finally, for the PDHG2-based scheme~\refEq{eq_supp:pdhg3b} yields $ y = - A^T\nabla H_f(Au)$, such that~\refEq{eq_supp:pdhg3b} yields the fixed-point~\refEq{eq_supp:fixedPoint}. Again, initializing $\bar{u}^0=u^0$ with the fixed-point and setting $y^0 = -A^T \nabla H_f(Au^0)$ results in a fixed-point of the PDHG2-based scheme and therefore yields the assertion.
    \end{proof}
    
    \textbf{Remark. } We would like to point out that the PDHG2 algorithm is closely related to ADMM: In fact, with an overrelaxation on the variable $y$, a reversed update order of $u$ and $y$, and $\tau = \frac{1}{\gamma}, ~ \theta = 1$, it is equivalent to the above ADMM algorithm in the convex case with proximity operators, see e.g. \cite{ChambollePockContOpti}, Section 5.3. Interestingly, one can show that this result still remains valid for our algorithmic schemes above in which the proximity operator has been replaced by a neural network. 
\clearpage

\section*{Evaluation}
\subsection*{Demosaicking}
We evaluated the effectiveness of our approach on noise free demosaicking of 18 Bayer filtered images of the \textit{McMaster} color image dataset, ~\cite{zhang2011_mcmaster}. For visualization purposes \refFig{fig:all_images_demosaicking} presents demosaicking results obtained with our approach applying the fixed denoising network trained on noise with standard deviation $\sigma = 0.02$. The images include a magnified area of the residual error which illustrates the varying demosaicking performance on differently structured parts of the image. In completion of Figure 4 of the main paper \refTab{table:all_psnrs_demosaicking} contains a comprehensive list of channel-wise PSNR values for each of the 18 color images. The superior reconstruction of the green channel can be attributed to its dominance in the \textit{RGGB} filter pattern. For a full comparison of our results with the state-of-the-art methods mentioned in the main paper we refer to the supplementary material of \cite{heide14_flexisp} and \cite{gharbi16_demosaicking}.


\begin{table}

\caption{Channel-wise PSNRs in [dB] for each Bayer filtered image of the \textit{McMaster} color image dataset. Our method applies the fixed denoising network trained on $\sigma = 0.02$.}
\label{table:all_psnrs_demosaicking}
\centering
\scriptsize
\setlength{\unitlength}{0.15\textwidth}

\newcommand{\row}[7]{
\multirow{3}{*}{\textit{#1}}
    &{R}&#2&\textbf{#5}\\
    \cline{2-4}
    &{G}&#3&\textbf{#6}\\
    \cline{2-4}
    &{B}&#4&\textbf{#7}\\ 
    \cline{2-4}
}

\resizebox{0.8\linewidth}{!}{%
\begin{tabular}{|c|c|c|c|}%
\hline
\multirow{2}{*}{Image} & \multirow{2}{*}{Channel} & \multicolumn{2}{c|}{Reconstruction PSNR in [\textit{dB}]} \\
                       \cline{3-4}
                       & & {FlexISP$^*$} & {Ours} \\
\hline
\row{1}{28.52}{31.55}{26.71}{29.09}{32.04}{27.01}
\hline
\row{2}{33.86}{38.39}{32.18}{34.69}{39.30}{32.85}
\hline
\row{3}{32.31}{35.56}{29.80}{34.33}{36.83}{30.81}
\hline
\row{4}{35.77}{39.90}{32.92}{38.55}{41.08}{34.47}
\hline
\row{5}{34.68}{37.30}{30.67}{35.31}{37.71}{31.65}
\hline
\row{6}{37.12}{41.69}{34.40}{39.38}{43.09}{36.44}
\hline
\row{7}{35.35}{38.31}{33.55}{35.89}{38.62}{33.85}
\hline
\row{8}{35.95}{40.35}{35.56}{38.42}{41.80}{37.18}
\hline
\row{9}{34.76}{40.74}{35.78}{36.78}{41.81}{36.86}
\hline
\row{10}{37.31}{41.61}{36.62}{37.57}{41.54}{36.90}
\hline
\row{11}{38.71}{41.23}{37.90}{39.92}{42.19}{38.54}
\hline
\row{12}{37.96}{40.52}{35.56}{38.46}{41.60}{37.22}
\hline
\row{13}{40.49}{44.74}{37.84}{42.46}{45.46}{38.68}
\hline
\row{14}{38.07}{42.65}{35.88}{39.13}{43.06}{36.25}
\hline
\row{15}{36.77}{42.34}{38.42}{37.26}{42.58}{38.90}
\hline
\row{16}{32.48}{34.05}{32.61}{34.16}{35.19}{32.65}
\hline
\row{17}{31.84}{36.57}{31.77}{33.37}{37.40}{32.30}
\hline
\row{18}{32.78}{36.15}{34.17}{34.02}{36.92}{35.09}
\hline
\hline
\row{AVG}{35.26}{39.09}{34.02}{\textbf{36.60}}{\textbf{39.90}}{\textbf{34.87}}
\hline
\hline
AVG & RGB & 36.12 & \textbf{37.12} \\
\hline
\end{tabular}
}
\vspace{-0.27cm}
\end{table}

\begin{figure*}

\centering
\contourlength{0.75pt}
\setlength{\fboxsep}{0pt}
\setlength{\fboxrule}{0.01\unitlength}
\setlength{\unitlength}{0.15\textwidth}

\newcommand{\img}[2]{
\begin{picture}(1,1) 
\ifarxiv
    \put(0,0){\includegraphics[width=\unitlength]{supp/all_images_demosaicking/#1_small}}
    \put(0.05,0.54){\fcolorbox{white}{white}{\includegraphics[trim={200 175 200 225},clip,width=0.4\unitlength]{supp/all_images_demosaicking/#1_diff_small}}}
\else
    \put(0,0){\includegraphics[width=\unitlength]{supp/all_images_demosaicking/#1}}
    \put(0.05,0.54){\fcolorbox{white}{white}{\includegraphics[trim={200 175 200 225},clip,width=0.4\unitlength]{supp/all_images_demosaicking/#1_diff}}}
\fi
\put(0.88,0.88){\makebox(0,0)[B]{\large \contour{black}{\textcolor{white}{#2}}}}  
\end{picture}
}

\resizebox{\textwidth}{!}{%
\begin{tabular}{*{6}{m{\unitlength}}}%

\img{1_ours}{1} & \img{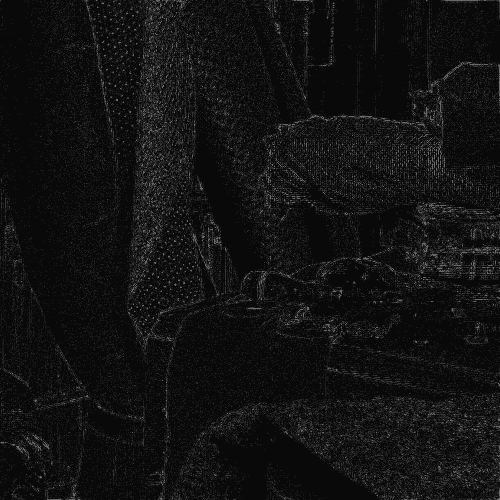}{2} & \img{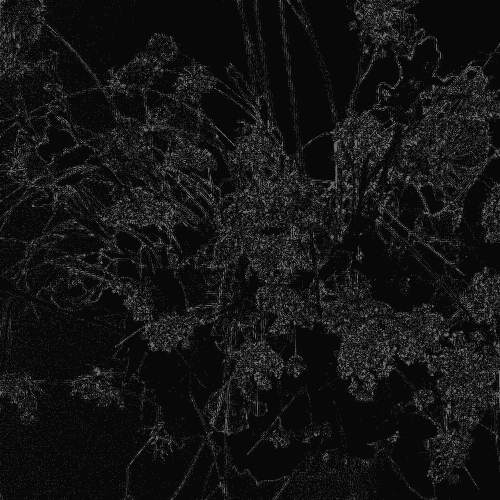}{3} & \img{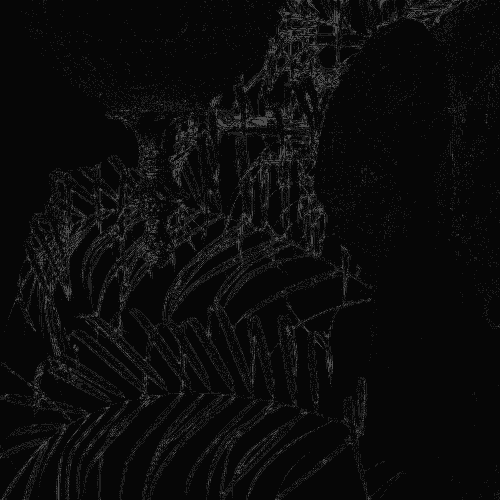}{4} & \img{5_ours}{5} & \img{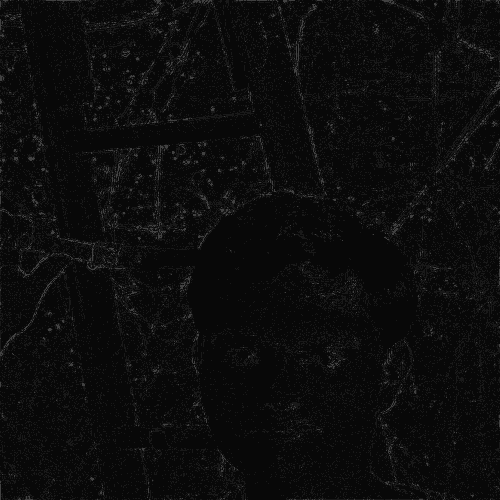}{6} \\
\img{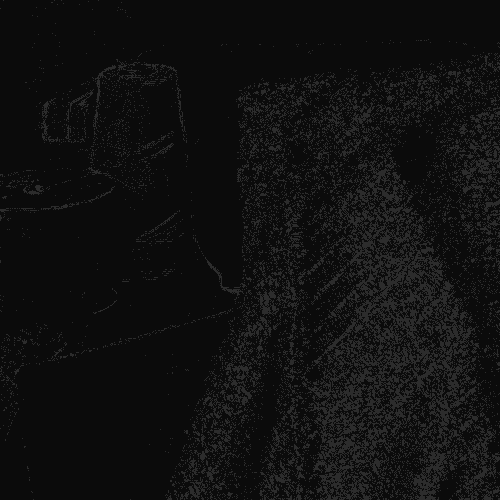}{7} & \img{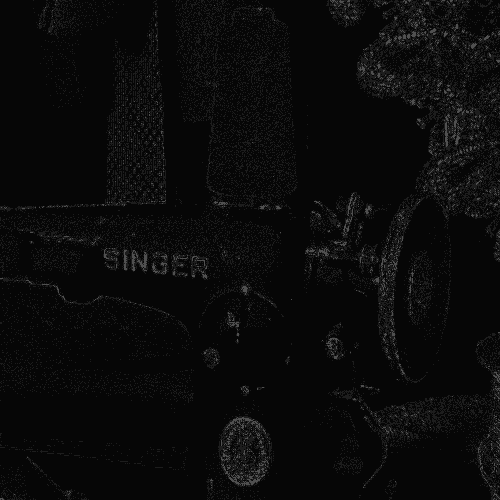}{8} & \img{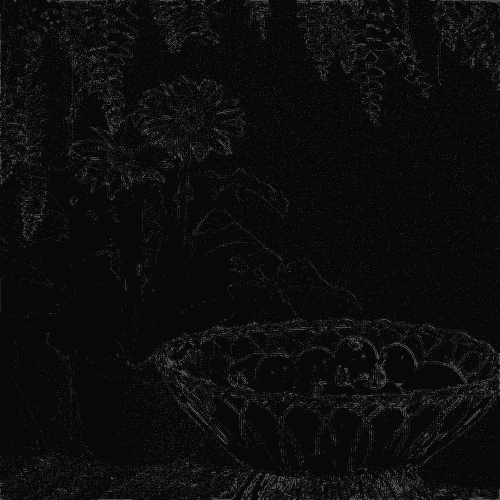}{9} & \img{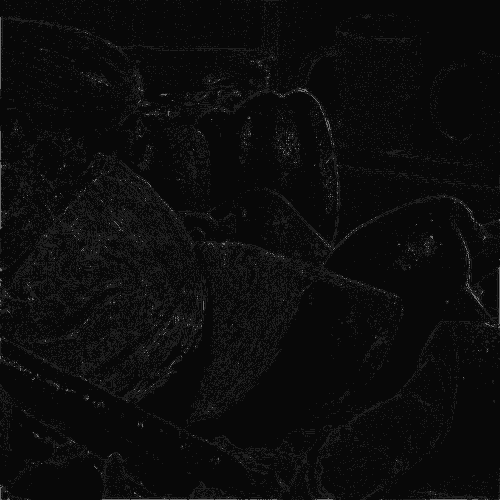}{10} & \img{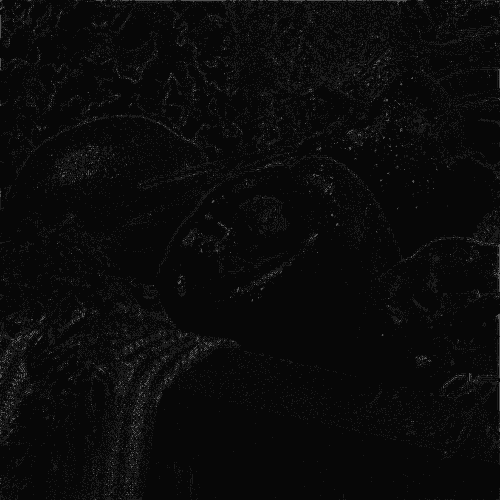}{11} & \img{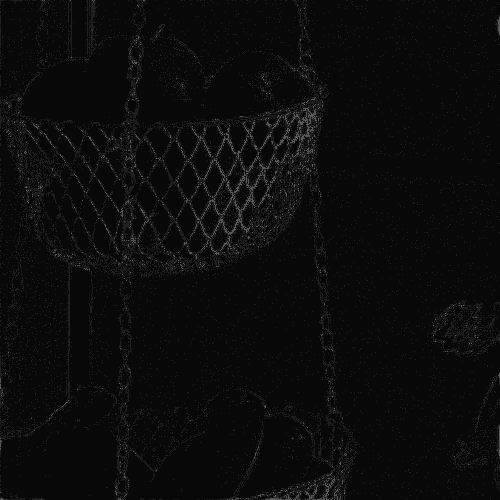}{12} \\
\img{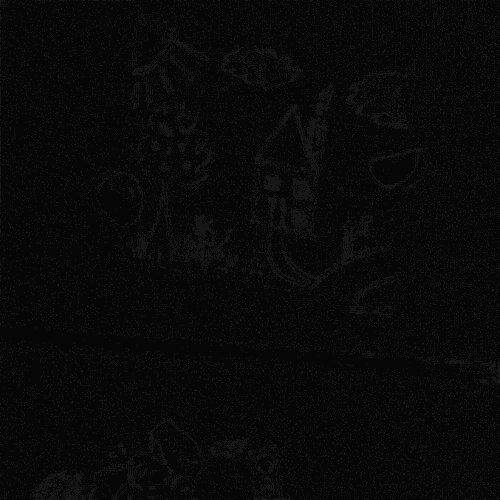}{13} & \img{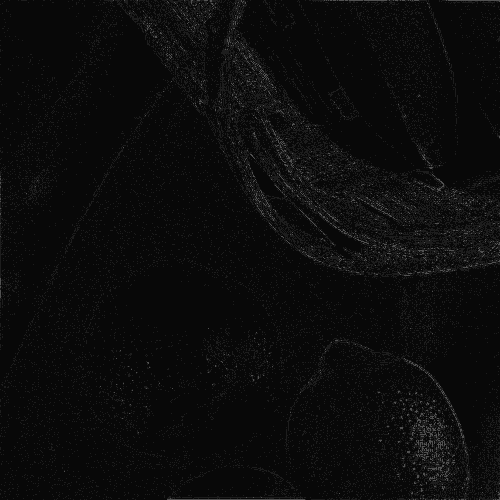}{14} & \img{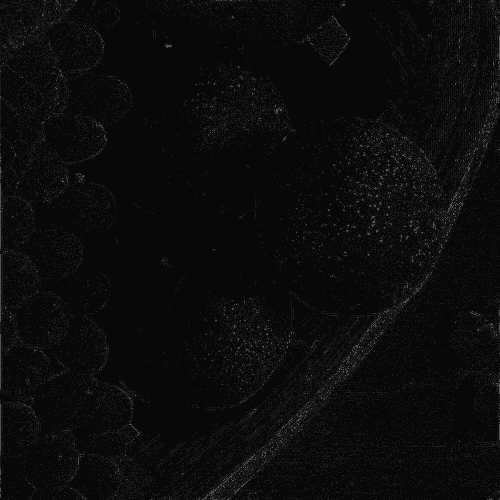}{15} & \img{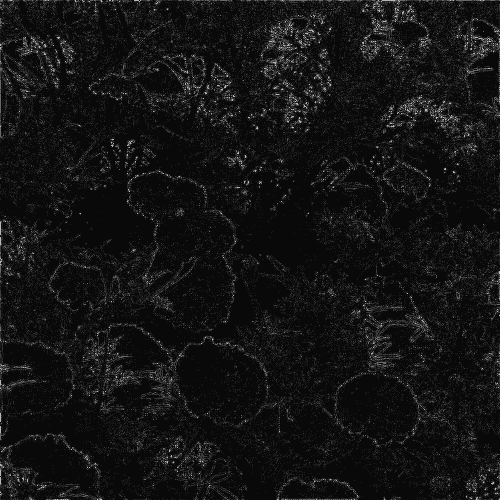}{16} & \img{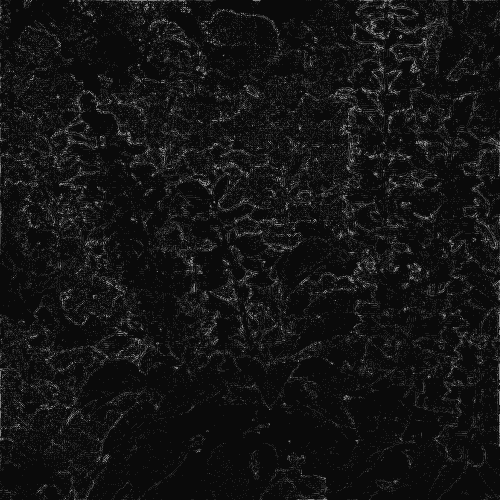}{17} & \img{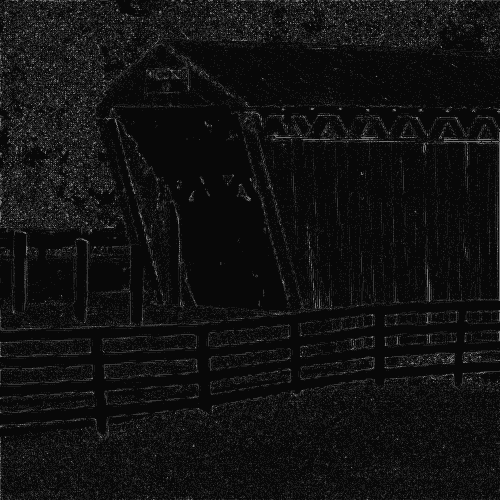}{18} \\

\end{tabular}
}

\caption{We demosaicked 18 images of the \textit{McMaster} color image dataset applying our approach with the fixed denoising network. To illustrate the remaining reconstruction error we added magnified residual images. To avoid boundary effects the images were cropped by $5$ pixels.}
\label{fig:all_images_demosaicking}
\vspace{-0.27cm}
\end{figure*}

\subsection*{Deconvolution}
Our experimental setup consists of the five (\textit{a} - \textit{e}) deconvolution experiments proposed in \cite{schuler13_mlp}. These experiments corrupt 11 standard test images with different blur kernels and Gaussian noise levels. \refFig{fig:all_images_deblurring} shows the corresponding dataset as well as exemplary deconvolution results obtained by our approach using the fixed network trained on noise with standard deviation $\sigma = 0.02$. The corresponding PSNR values as well as our FlexISP$^*$ results are presented in \refTab{table:all_psnrs_deblurring}. A detail explanation of FlexISP$^*$, our reimplementation of \cite{heide14_flexisp}, can be found in the main paper. To illustrate the robustness with respect to the choice of network we also included the results for networks trained on different $\sigma$. For a comprehensive comparison with the methods mentioned in the paper we again refer to the supplementary material of \cite{heide14_flexisp}. For the sake of reproducibility \refTab{table:evaluation_deblurring_params} includes the results of our grid search for the data fidelity parameter $\alpha$ as well as for the regularization parameter $\beta_{TV}$ for multiple networks.


\begin{table*}

\caption{The optimal deblurring values for the data fidelity parameter $\alpha$ as well as for the regularization parameter $\beta_{TV}$ with respect to our method applying denoising networks trained on different noise standard deviations $\sigma$. All values were obtained by performing and extensive grid search of the parameter space. Following Proposition~3.2 we set the dual step size of the PDHG algorithm to $\gamma = 1.0$ and determined the primal step size $\tau$ from $\tau \gamma < c$ with $c$ being the squared norm of the involved linear operator.}
\label{table:evaluation_deblurring_params}

\centering

\begin{tabular}{|l *{10}{|c} |}
\hline
\multirow{2}{*}{$\sigma$}           & \multicolumn{2}{c|}{Experiment \textit{a}}
                                    & \multicolumn{2}{c|}{Experiment \textit{b}}
                                    & \multicolumn{2}{c|}{Experiment \textit{c}}
                                    & \multicolumn{2}{c|}{Experiment \textit{d}}
                                    & \multicolumn{2}{c|}{Experiment \textit{e}}\\
                                    \cline{2-11}
                                    & {$\alpha$} & {$\beta_{TV}$}
                                    & {$\alpha$} & {$\beta_{TV}$}
                                    & {$\alpha$} & {$\beta_{TV}$}
                                    & {$\alpha$} & {$\beta_{TV}$}
                                    & {$\alpha$} & {$\beta_{TV}$}\\
\hline
\hline0.01& 1 & 0.00 & 25 & 0.00 & 40 & 0.05 & 250 & 0.01 & 10 & 0.00 \\
\hline0.02& 2 & 0.00 & 75 & 0.00 & 4 & 0.00 & 73 & 0.00 & 23 & 0.00 \\
\hline0.03& 5 & 0.00 & 149 & 0.00 & 7 & 0.00 & 107 & 0.00 & 43 & 0.00 \\
\hline0.04& 7 & 0.00 & 200 & 0.00 & 10 & 0.00 & 140 & 0.00 & 64 & 0.00 \\
\hline0.05& 11 & 0.01 & 160 & 0.01 & 13 & 0.00 & 200 & 0.00 & 93 & 0.00 \\
\hline0.06& 13 & 0.00 & 200 & 0.01 & 17 & 0.00 & 240 & 0.00 & 120 & 0.00 \\
\hline0.07& 16 & 0.00 & 424 & 0.00 & 24 & 0.00 & 272 & 0.00 & 150 & 0.00 \\
\hline0.08& 23 & 0.00 & 467 & 0.00 & 34 & 0.00 & 467 & 0.00 & 200 & 0.00 \\
\hline0.09& 24 & 0.00 & 300 & 0.01 & 36 & 0.00 & 600 & 0.00 & 267 & 0.00 \\
\hline0.20& 100 & 0.00 & 800 & 0.03 & 150 & 0.00 & 2400 & 0.00 & 480 & 0.10 \\
\hline
\end{tabular}


\end{table*}

\begin{figure*}[t]

\centering

\setlength{\unitlength}{0.1\textwidth}
\setlength{\fboxsep}{0pt}%
\setlength{\fboxrule}{0.12\unitlength}
\newcolumntype{M}[1]{>{\centering\arraybackslash}m{#1}}
\newcommand{\img}[1]{\includegraphics[width=\unitlength]{supp/#1}}

\newcommand{\row}[2]{
\multirow{2}{*}{\rotatebox[origin=lt]{90}{Experiment \textit{#1}}} & #2 &\img{all_images_deblurring/experiment_#1_barbara_blurred}&\img{all_images_deblurring/experiment_#1_boat_blurred}&\img{all_images_deblurring/experiment_#1_cameraman_blurred}&\img{all_images_deblurring/experiment_#1_couple_blurred}&\img{all_images_deblurring/experiment_#1_fingerprint_blurred}&\img{all_images_deblurring/experiment_#1_hill_blurred}&\img{all_images_deblurring/experiment_#1_house_blurred}&\img{all_images_deblurring/experiment_#1_lena_blurred}&\img{all_images_deblurring/experiment_#1_man_blurred}&\img{all_images_deblurring/experiment_#1_montage_blurred}&\img{all_images_deblurring/experiment_#1_peppers_blurred}
\\
& &\img{all_images_deblurring/experiment_#1_barbara_ours}&\img{all_images_deblurring/experiment_#1_boat_ours}&\img{all_images_deblurring/experiment_#1_cameraman_ours}&\img{all_images_deblurring/experiment_#1_couple_ours}&\img{all_images_deblurring/experiment_#1_fingerprint_ours}&\img{all_images_deblurring/experiment_#1_hill_ours}&\img{all_images_deblurring/experiment_#1_house_ours}&\img{all_images_deblurring/experiment_#1_lena_ours}&\img{all_images_deblurring/experiment_#1_man_ours}&\img{all_images_deblurring/experiment_#1_montage_ours}&\img{all_images_deblurring/experiment_#1_peppers_ours}\\\\
}

\resizebox{\textwidth}{!}{%
\begin{tabular}{m{0.05\unitlength} *{12}{m{\unitlength}}}%
&{Blur kernel} & \textit{Barbara} & \textit{Boat} & \textit{Cameraman} & \textit{Couple} & \textit{Fingerprint} & \textit{Hill} & \textit{House} & \textit{Lena} & \textit{Man} & \textit{Montage} & \textit{Peppers} \\
\row{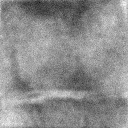}{\img{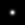}}
\row{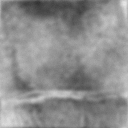}{\img{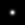}}
\row{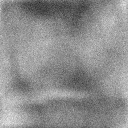}{\img{blur_kernel_b}}
\row{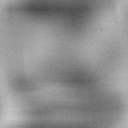}{\fcolorbox{black}{white}{\includegraphics[width=0.76\unitlength]{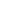}}}
\row{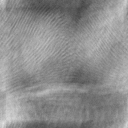}{\img{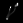}}
\end{tabular}
}

\caption{Our deconvolution dataset based on the experiments introduced in~\cite{schuler13_mlp}. Each image ($128\times 128$ pixels) is shown in its corrupted as well as by our approach reconstructed version. The deblurring was performed using the fixed denoising network trained on $\sigma = 0.02$. To avoid boundary effects the images were cropped by $12$ pixels. For visualization purposes we show enlarged versions of the different blur kernels.}
\label{fig:all_images_deblurring}
\vspace{-0.27cm}
\end{figure*}
\begin{table*}[t]

\caption{Imagewise PSNRs in [dB] for each of our 5 (\textit{a} - \textit{e}) deconvolution experiments for FlexISP$^*$ and multiple versions of our approach using denoising networks trained on different $\sigma$. Our application independent approach applied a network  trained on $\sigma = 0.02$.}

\label{table:all_psnrs_deblurring}
\centering

\setlength{\unitlength}{0.1\textwidth}
\newcolumntype{M}[1]{>{\centering\arraybackslash}m{#1}}

\resizebox{\textwidth}{!}{%
\begin{tabular}{c|c| *{11}{c |}}%
\cline{2-13}
\multirow{2}{*}{} &\multirow{2}{*}{Method} & \multicolumn{11}{c|}{Reconstruction PSNR in [\textit{dB}]} \\
                        \cline{3-13}
                        & & \textit{Barbara} & \textit{Boat} & \textit{Cameraman} & \textit{Couple} & \textit{Fingerprint} & \textit{Hill} & \textit{House} & \textit{Lena} & \textit{Man} & \textit{Montage} & \textit{Peppers} \\

\cline{2-13}

\multirow{7}{*}{\rotatebox[origin=lt]{90}{Experiment \textit{a}}}
    &FlexISP$^*$~\cite{heide14_flexisp}&25.93           &\textbf{24.44} &23.65         &\textbf{24.16}&\textbf{17.43}&25.83         &26.93        &25.05         &24.90         &22.84         &26.41\\
    &Ours                              & \textbf{26.27} &24.41          &\textbf{23.78}&24.15         &17.41         &\textbf{25.89}&\textbf{27.35}&\textbf{25.34}&\textbf{25.02}&\textbf{23.00}&\textbf{26.99}\\
    \cline{2-13}
    &Ours,{\setlength{\thickmuskip}{0mu} $\sigma= 0.01$}&25.97&24.34&23.40&24.13&17.41&25.78&26.53&24.95&24.88&22.89&26.49\\
    &Ours,{\setlength{\thickmuskip}{0mu} $\sigma= 0.04$}&26.19&24.48&23.93&24.26&17.43&25.95&27.38&25.42&25.12&22.97&27.06\\
    &Ours,{\setlength{\thickmuskip}{0mu} $\sigma= 0.06$}&26.32&24.46&23.97&24.27&17.44&25.98&27.56&25.51&25.13&23.02&27.12\\
    &Ours,{\setlength{\thickmuskip}{0mu} $\sigma= 0.09$}&26.27&24.42&23.99&24.27&17.44&26.03&27.03&25.60&25.17&23.06&27.04\\
    &Ours,{\setlength{\thickmuskip}{0mu} $\sigma= 0.20$}&26.17&24.31&23.79&24.17&17.43&25.76&27.32&25.50&25.03&22.85&26.95\\

\cline{2-13}

\multirow{7}{*}{\rotatebox[origin=lt]{90}{Experiment \textit{b}}}
    &FlexISP$^*$~\cite{heide14_flexisp}&29.14         &26.62         &26.00         &26.55         &17.81         &28.70         &30.99    &27.90      &27.38\        &24.47&29.72\\
    &Ours                              &\textbf{29.38}&\textbf{26.74}&\textbf{26.26}&\textbf{26.70}&\textbf{17.86}&\textbf{28.81}&\textbf{31.43}&\textbf{28.27}&\textbf{27.58}&\textbf{24.70}&\textbf{30.13}\\
    \cline{2-13}
    &Ours,{\setlength{\thickmuskip}{0mu} $\sigma= 0.01$}&29.36&26.66&26.05&26.64&17.82&28.87&31.24&28.17&27.60&24.55&30.19\\
    &Ours,{\setlength{\thickmuskip}{0mu} $\sigma= 0.04$}&29.40&26.70&26.28&26.71&17.85&28.82&31.52&28.40&27.64&24.66&30.14\\
    &Ours,{\setlength{\thickmuskip}{0mu} $\sigma= 0.06$}&29.52&26.79&26.37&26.74&17.83&28.91&31.39&28.37&27.74&24.62&30.27\\
    &Ours,{\setlength{\thickmuskip}{0mu} $\sigma= 0.09$}&29.49&26.77&26.32&26.70&17.84&28.86&31.60&28.39&27.72&24.51&30.24\\
    &Ours,{\setlength{\thickmuskip}{0mu} $\sigma= 0.20$}&29.13&26.33&25.17&26.42&17.75&28.51&30.63&27.98&27.38&23.80&29.79\\

\cline{2-13}

\multirow{7}{*}{\rotatebox[origin=lt]{90}{Experiment \textit{c}}}
    &FlexISP$^*$~\cite{heide14_flexisp}&\textbf{23.24}&\textbf{22.11}&\textbf{21.01}&\textbf{22.04}&\textbf{17.04}&23.05         &\textbf{23.57}&\textbf{22.57}&22.43&\textbf{21.38}&\textbf{23.47}\\
    &Ours                              &23.12         &22.01         &20.85         &21.93         &17.02         &\textbf{23.12}&22.77&22.43&\textbf{22.49}&21.22&23.19\\
    \cline{2-13}
    &Ours,{\setlength{\thickmuskip}{0mu} $\sigma= 0.01$}&22.49&21.77&20.96&21.75&17.07&22.83&22.64&22.02&22.27&20.91&22.51\\
    &Ours,{\setlength{\thickmuskip}{0mu} $\sigma= 0.04$}&23.03&22.18&21.20&21.89&17.03&23.12&23.26&22.64&22.41&21.43&23.29\\
    &Ours,{\setlength{\thickmuskip}{0mu} $\sigma= 0.06$}&23.02&22.23&21.27&21.98&17.06&23.18&23.62&22.51&22.49&21.40&23.56\\
    &Ours,{\setlength{\thickmuskip}{0mu} $\sigma= 0.09$}&23.15&22.20&21.19&21.93&17.09&23.12&23.50&22.28&22.53&21.33&23.45\\
    &Ours,{\setlength{\thickmuskip}{0mu} $\sigma= 0.20$}&23.07&22.21&21.42&21.97&17.06&23.04&23.20&22.48&22.63&21.32&23.57\\

\cline{2-13}

\multirow{7}{*}{\rotatebox[origin=lt]{90}{Experiment \textit{d}}}
    &FlexISP$^*$~\cite{heide14_flexisp}&\textbf{23.13}&\textbf{22.92}&\textbf{21.92}&\textbf{22.87}&\textbf{17.44}&\textbf{23.88}&\textbf{24.95}&\textbf{22.57}&\textbf{22.33}&\textbf{22.19}&\textbf{23.59}\\
    &Ours                              &22.48         &22.45         &20.89         &22.69         &17.38         &23.53         &23.37&22.22&21.97&21.64&22.90\\
    \cline{2-13}
    &Ours,{\setlength{\thickmuskip}{0mu} $\sigma= 0.01$}&21.81&22.08&20.71&22.40&17.25&22.98&23.01&21.52&21.62&21.30&22.03\\
    &Ours,{\setlength{\thickmuskip}{0mu} $\sigma= 0.04$}&22.97&22.66&21.78&22.77&17.37&23.78&24.91&22.51&22.23&22.07&23.33\\
    &Ours,{\setlength{\thickmuskip}{0mu} $\sigma= 0.06$}&23.21&22.71&21.83&22.81&17.39&23.87&25.57&22.71&22.39&22.19&23.70\\
    &Ours,{\setlength{\thickmuskip}{0mu} $\sigma= 0.09$}&23.19&22.76&21.85&22.81&17.37&23.87&25.48&22.63&22.39&22.64&23.66\\
    &Ours,{\setlength{\thickmuskip}{0mu} $\sigma= 0.20$}&31.42&29.28&30.50&28.78&23.80&29.57&33.06&30.73&29.24&31.29&31.94\\

\cline{2-13}

\multirow{7}{*}{\rotatebox[origin=lt]{90}{Experiment \textit{e}}}
    &FlexISP$^*$~\cite{heide14_flexisp}&30.60&28.54&29.19&28.27&\textbf{23.59}&29.31&32.65&29.93&28.49&30.63&31.13\\
    &Ours                              &\textbf{31.67}&\textbf{29.24}&\textbf{30.84}&\textbf{28.85}&23.42&\textbf{29.69}&\textbf{33.38}&\textbf{30.80}&\textbf{29.15}&\textbf{32.45}&\textbf{32.36}\\
    \cline{2-13}
    &Ours,{\setlength{\thickmuskip}{0mu} $\sigma= 0.01$}&29.86&28.69&29.14&28.02&22.19&29.28&31.28&29.36&28.35&29.70&30.65\\
    &Ours,{\setlength{\thickmuskip}{0mu} $\sigma= 0.04$}&31.75&29.51&30.99&28.88&24.20&29.80&33.65&30.93&29.37&32.49&32.28\\
    &Ours,{\setlength{\thickmuskip}{0mu} $\sigma= 0.06$}&31.73&29.48&30.80&28.85&24.14&29.75&33.37&30.91&29.40&32.22&32.21\\
    &Ours,{\setlength{\thickmuskip}{0mu} $\sigma= 0.09$}&31.42&29.28&30.50&28.78&23.80&29.57&33.06&30.73&29.24&31.29&31.94\\
    &Ours,{\setlength{\thickmuskip}{0mu} $\sigma= 0.20$}&28.33&25.86&25.42&25.31&18.37&27.63&27.78&27.10&26.48&24.39&27.08\\

\cline{2-13}
\end{tabular}
}
\vspace{-0.27cm}
\end{table*}

\fi

\end{document}